%% file: main.tex
\documentclass[11pt, letterpaper]{article}

\usepackage[round]{natbib}
\let\cite\citep

\usepackage{fullpage}
\usepackage{dirtytalk}

\input{commands}

\definecolor{darkscarlet}{rgb}{0.34, 0.01, 0.1}
\definecolor{yaleblue}{rgb}{0.06, 0.3, 0.57}
\definecolor{darkpowderblue}{rgb}{0.0, 0.2, 0.6}

\usepackage[pagebackref=true]{hyperref}
\hypersetup{
    colorlinks=true,
    citecolor=darkscarlet,
    linkcolor=darkpowderblue,
    filecolor=magenta,      
    urlcolor=yaleblue,
    menucolor=gray
}

\renewcommand*{\backrefalt}[4]{%
    \ifcase #1 \footnotesize{(Not cited.)}%
    \or        \footnotesize{(Cited on page~#2)}%
    \else      \footnotesize{(Cited on pages~#2)}%
    \fi}

\title{\textbf{Towards a Better Theoretical Understanding of Independent Subnetwork Training}}
\date{}

\author{%
    Egor Shulgin \qquad Peter Richtárik \\
    \phantom{x}
    \\
    King Abdullah University of Science and Technology (KAUST) \\
    Thuwal, Saudi Arabia
}

\begin{document}

\maketitle

\begin{abstract}
Modern advancements in large-scale machine learning would be impossible without the paradigm of data-parallel distributed computing. Since distributed computing with large-scale models imparts excessive pressure on communication channels, significant recent research has been directed toward co-designing communication compression strategies and training algorithms with the goal of reducing communication costs. While pure data parallelism allows better data scaling, it suffers from poor model scaling properties. Indeed, compute nodes are severely limited by memory constraints, preventing further increases in model size. For this reason, the latest achievements in training giant neural network models also rely on some form of model parallelism. In this work, we take a closer theoretical look at Independent Subnetwork Training (IST), which is a recently proposed and highly effective technique for solving the aforementioned problems. We identify fundamental differences between IST and alternative approaches, such as distributed methods with compressed communication, and provide a precise analysis of its optimization performance on a quadratic model.
\end{abstract}

\section{Introduction}

A huge part of today's machine learning success is driven by the possibility of building more and more complex models and training them on increasingly larger datasets. This rapid progress has become feasible due to advancements in distributed optimization, which is necessary for proper scaling when the size of the training data grows~\cite{zinkevich2010parallelized}. In a typical scenario, data parallelism is used for efficiency and implies sharding the dataset across computing devices. This allowed very efficient scaling and acceleration of training moderately sized models by using additional hardware \cite{goyal2018accurate}. However, this data parallel approach can suffer from communication bottleneck, which has sparked extensive research on distributed optimization with compressed communication of the parameters between nodes \cite{alistarh2017qsgd, konecny2017federated, seide20141}.

\subsection{The need for model parallelism}
Despite its efficiency, data parallelism has some fundamental limitations when it comes to scaling up the model size. As the dimensions of a model increase, the amount of memory required to store and update the parameters also increases, which becomes problematic due to resource constraints on individual devices. This has led to the development of model parallelism \cite{dean2012large,Hydra}, which splits a large model across multiple nodes, with each node responsible for computations of parts of the model \cite{farber1997parallel, zhang1989efficient}. 
However, naive model parallelism also poses challenges because each node can only update its portion of the model based on the data it has access to. This creates a need for very careful management of communication between devices. Thus, a combination of both data and model parallelism is often necessary to achieve efficient and scalable training of huge models.

\textbf{IST.}
Independent Subnetwork Training (IST) is a technique that suggests dividing a neural network into smaller subparts, training them in a distributed parallel fashion, and then aggregating the results to update the weights of the whole model. In IST, every subnetwork can operate independently and has fewer parameters than the full model, which not only reduces the load on computing nodes but also results in faster synchronization. A generalized analog of the described method is formalized as an iterative procedure in Algorithm \ref{alg:IST} and schematically depicted in Figure \ref{fig:IST_schematic}.
IST paradigm was pioneered by \citet{yuan2022distributed} for networks with fully connected layers and was later extended to ResNets \cite{dun2022resist} and Graph architectures \cite{wolfe2021gist}. Previous experimental studies have shown that IST is a very promising approach for various applications as it allows to effectively combine data and model parallelism and train larger models with limited compute. 
In addition, \citet{liao2022on} performed theoretical analysis of IST for overparameterized single hidden layer neural networks with ReLU activations. The idea of IST was also recently extended to the federated setting via an asynchronous distributed dropout technique \cite{dun2023efficient}.

\textbf{Federated Learning.} Another important setting when the data is distributed (due to privacy reasons) is Federated Learning \cite{kairouz2021advances, konecny2017federated, mcmahan2017communication}. In this scenario, computing devices are often heterogeneous and more resource-constrained \cite{caldas2018expanding} (e.g.~mobile phones) in comparison to data-center settings. Such challenges have prompted extensive research efforts into selecting smaller and more efficient submodels for local on-device training \cite{alam2022fedrolex, charles2022federated, chen2022fedobd, diao2020heterofl, horvath2021fjord, jiang2022model, lin2022federated, qiu2022zerofl, wen2022federated, yang2022partial}. 
Many of these works propose approaches to adapt submodels, often tailored to specific neural network architectures, based on the capabilities of individual clients for various machine learning tasks. 
However, there is a lack of comprehension regarding the theoretical properties of these methods.

\subsection{Summary of contributions}
After reviewing the literature, we identified a glaring gap in the rigorous understanding of IST convergence, directly motivating our research. The main contributions of this paper include: 
\begin{itemize}
    \item A novel approach to analyzing distributed methods that combine data and model parallelism by operating with sparse submodels for a quadratic model.

    \item The first analysis of independent subnetwork training in homogeneous and heterogeneous scenarios without restrictive assumptions on gradient estimators.
    
    \item Identification of the settings when IST can optimize very efficiently or not converge to the optimal solution but only to an irreducible neighborhood that is also tightly characterized.
    
    \item Empirical validation of the proposed theory through experiments for several practical settings.
\end{itemize}

\section{Formalism and setup} \label{sec:setup}

\begin{algorithm*}[tb]
\begin{algorithmic}[1] 
\caption{Distributed Submodel (Stochastic) Gradient Descent} \label{alg:IST}
    \State \textbf{Parameters:} learning rate $\gamma>0$; sketches $\mC_1, \ldots, \mC_n$; initial model $x^0 \in \bR^d$
    \For{$k = 0, 1, 2 \ldots$}
    \State Select submodels $w_i^k = \mC_i^k x^k$ for $i \in [n]$ and broadcast to all computing nodes
    \For{$i = 1, \ldots, n$ in parallel}
    \State Compute local (stochastic) gradient w.r.t. submodel: $\mC_i^k \nabla f_i(w_i^k)$
    \State Take (maybe multiple) gradient descent step $z_i^+ = w_i^k - \gamma \mC_i^k \nabla f_i(w_i^k)$
    \State Send $z_i^+$ to the server 
    \EndFor
    \State Aggregate/merge received submodels: $x^{k+1} = \frac{1}{n} \sum_{i=1}^n z_i^+$
    \EndFor
\end{algorithmic}
\end{algorithm*}

We consider the standard optimization formulation of a distributed/federated learning problem \cite{wang2021field}
\begin{equation} \label{eq:general_problem}
  \min \limits_{x \in \bR^d}~\sbr{f(x) \eqdef \frac{1}{n} \sum \limits_{i=1}^n f_i(x)},
\end{equation}
where $n$ is the number of clients/workers, and each $f_i: \bR^d \to \bR$ represents the loss of the model parameterized by vector $x \in \bR^d$ on the data of client $i$.

A typical Stochastic Gradient Descent (SGD)-type method for solving this problem has the form
\begin{equation} \label{eq:SGD_generic} 
  x^{k+1} = x^k - \gamma g^k, \qquad g^k = \frac{1}{n}\sum \limits_{i=1}^n g_i^k,
\end{equation}
where $\gamma>0$ is the stepsize and $g_i^k$ is a suitably constructed estimator of $\nabla f_i(x^k)$. In the distributed setting, computation of gradient estimators $g_i^k$ is typically performed by clients, and the results are sent to the server, which subsequently performs aggregation via averaging $g^k = \frac{1}{n}\sum_{i=1}^n g_i^k$. The average is then used to update the model $x^{k+1}$ via a gradient-type method \eqref{eq:SGD_generic}, and at the next iteration, the model is broadcasted back to the clients. The process is repeated iteratively until a suitable model is found.

One of the main techniques used to accelerate distributed training is lossy \textit{communication compression} \cite{alistarh2017qsgd, konecny2017federated, seide20141},
which suggests applying a (possibly randomized) lossy compression mapping $\cC$ to a vector/matrix/tensor $x$ before broadcasting. This reduces the bits sent per communication round at the cost of transmitting a less accurate estimate $\cC(x)$ of $x$. 
Described technique can be formalized in the following definition.

\begin{definition}[Unbiased compressor] \label{def:unbiased_compressor}
A randomized mapping $\cC: \bR^d \to \bR^d$ is an {\bf unbiased compression operator} ($\cC \in \bU(\omega)$ for brevity) if for some $\omega \geq 0$ and $\forall x \in \bR^d$
\begin{equation} \label{eq:unbiased_compressor}
    \Exp{\cC(x)} = x, \qquad \Exp{\norm{\cC(x) - x}^2} \leq \omega \norm{x}^2.
\end{equation}
\end{definition}
A notable example of a mapping from this class is the \textit{random sparsification} (\texttt{Rand-q} for $q \in \{1, \dots, d\}$) operator defined by
\begin{equation} \label{eq:rand-k}
    \cC_{\texttt{Rand-q}}(x) \eqdef \mC_q x = \frac{d}{q} \sum\limits_{i\in S} e_i e_i^\top x ,
\end{equation}
where $e_1,\dots,e_d \in\bR^d$ are standard unit basis vectors, and $S$ is a random subset of $[d] \eqdef \{1, \dots, d\}$ sampled from the uniform distribution on the all subsets of $[d]$ with cardinality $q$. \texttt{Rand-q} belongs to $\bU\br{d/q-1}$, which means that the more elements are \say{dropped} (lower $q$), the higher the variance $\omega$ of the compressor.

In this work, we are mainly interested in a somewhat more general class of operators than mere sparsifiers. In particular, we are interested in compressing via the application of random matrices, i.e., via {\em sketching}. A sketch $\mC_i^k \in \R^{d\times d}$ can be used to represent submodel computations in the following way:
\begin{equation*}
    g_i^k \eqdef \mC_i^k \nabla f_i(\mC_i^k x^k),
\end{equation*}
where we require $\mC_i^k$ to be a symmetric positive semi-definite matrix. Such gradient estimates correspond to computing the local gradient with respect to a sparse submodel $\mC_i^k x^k$, and additionally sketching the resulting gradient with the same matrix $\mC_i^k$ to guarantee that the resulting update lies in the lower-dimensional subspace.

Using this notion, Algorithm \ref{alg:IST} (with one local gradient step) can be represented as:
\begin{equation} \label{eq:IST} 
    x^{k+1} = \frac{1}{n} \sum \limits_{i=1}^n \sbr{\mC_i^k x^k - \gamma \mC_i^k \nabla f_i(\mC_i^k x^k)},
\end{equation}
which is equivalent to the SGD-type update \eqref{eq:SGD_generic} when the
\emph{perfect reconstruction} property holds  (with probability one)
$$\mC^k \eqdef \frac{1}{n}\sum \limits_{i=1}^n \mC_i^k = \mathbf{I},$$ where $\mathbf{I}$ is the identity matrix.
This property is inherent for a specific class of compressors that are particularly useful for capturing the concept of an \textit{independent} subnetwork partition.

\begin{definition}[Permutation sketch] \label{def:permutation}
Assume that model size is greater than the number of clients $d \geq n$ and $d=q n$, where $q \geq 1$ is an integer\footnote{While this condition may look restrictive, it naturally holds for distributed learning in a data-center setting. Permutation sparsifiers were introduced by \cite{szlendak2022permutation} and generalized to other scenarios (like $n \geq d$).}.
Let $\pi=(\pi_1, \ldots, \pi_d)$ be a random permutation of $[d]$. Then for all $i \in [n]$, we define $\texttt{Perm-q}$ operator
\begin{equation} \label{eq:permutation}
    \mC_i \eqdef n \cdot \sum \limits_{j=q(i-1)+1}^{q i} e_{\pi_j} e_{\pi_j}^\top.
\end{equation}
\end{definition}
$\texttt{Perm-q}$ is unbiased and can be conveniently used for representing a structured decomposition of the model, such that every client $i$ is responsible for computations over a submodel $\mC_i x^k$.

Our convergence analysis relies on the assumption that was previously used for coordinate descent-type methods.
\begin{assumption}[Matrix smoothness] \label{ass:matrix-smoothness}
 A differentiable function $f: \bR^d \to \bR$ is $\mathbf{L}$-smooth, if there exists a positive semi-definite matrix $\mL \in \bR^{d \times d}$ such that
\begin{equation}\label{eq:L-matrix-smooth}
   f(x + h) \leq f(x) + \left\langle\nabla f(x), h \right\rangle + \frac{1}{2} \left\langle \mL h, h \right\rangle, \qquad \forall x, h \in \Rd.
\end{equation}
\end{assumption}
A standard $L$-smoothness condition is obtained as a special case of \eqref{eq:L-matrix-smooth} for $\mL = L \cdot \mI$. 
Matrix smoothness was previously used for designing data-dependent gradient sparsification to accelerate optimization in communication-constrained settings \cite{safaryan2021smoothness, wang2022theoretically}.

\subsection{Issues with existing approaches}

Consider the simplest gradient descent method with a compressed model in the single-node setting:
\begin{equation} \label{eq:CMGD} 
    x^{k+1} = x^k - \gamma \nabla f (\cC(x^k)).
\end{equation}

Algorithms belonging to this family require a different analysis in comparison to SGD \cite{gorbunov2020unified, gower2019sgd}, Distributed Compressed Gradient Descent \cite{alistarh2017qsgd, khirirat2018distributed}, and Randomized Coordinate Descent \cite{nesterov2012efficiency, richtarik2014iteration}-type methods because the gradient estimator is no longer unbiased 
\begin{equation*}
    \Exp{\nabla f (\cC(x))} \neq \nabla f(x) = \Exp{\cC(\nabla f(x))}.
\end{equation*}

This is why such kind of algorithms \eqref{eq:CMGD} are harder to analyze. So, prior results for \textit{unbiased} SGD \cite{khaled2022better} cannot be directly reused.
Furthermore, the nature of the bias in this type of gradient estimator does not exhibit additive noise, thereby preventing the application of previous analyses for biased SGD \cite{ajalloeian2020convergence}.

An assumption like the bounded stochastic gradient norm extensively used in previous works \cite{lin2019dynamic, zhou2022convergence} hinders an accurate understanding of such methods. This assumption hides the fundamental difficulty of analyzing a biased gradient estimator: 
\begin{equation} \label{eq:bounded_gradient}
    \Exp{\sqN{\nabla f(\cC(x))}} \leq G
\end{equation}
and may not hold, even for quadratic functions $f(x) = x^\top \mA x$. In addition, in the distributed setting, such a condition can result in vacuous bounds \cite{khaled2020tighter} as it does not capture heterogeneity accurately.

\subsection{Simplifications taken}

To conduct a thorough theoretical analysis of methods that combine data with model parallelism, we simplify the algorithm and problem setting to isolate the unique effects of this approach. The following considerations are made:
\begin{enumerate}
    \item[\textbf{(a)}] We assume that every node $i$ computes the true gradient at the submodel $\mC_i \nabla f_i(\mC_i x^k)$. 
    \item[\textbf{(b)}] A notable difference compared to the original IST Algorithm \ref{alg:IST} is that workers perform a single gradient descent step (or just gradient computation). 
    \item[\textbf{(c)}] Finally, we consider a special case of a quadratic model \eqref{eq:het_gen_quad_problem} as a loss function \eqref{eq:general_problem}.
\end{enumerate}

Condition \textbf{(a)} is mainly for the sake of simplicity and clarity of exposition and can be generalized to stochastic gradient estimators with bounded variance.
Condition \textbf{(b)} is imposed because local steps did not bring any theoretical efficiency improvements for heterogeneous settings until very recently \cite{mishchenko2022proxskip}, and even then, only with the introduction of additional control variables, which goes against the requirements of resource-constrained device settings. The reason behind \textbf{(c)} is that despite its apparent simplicity, the quadratic problem has been used extensively to study properties of neural networks \cite{zhang2019algorithmic, zhu2022quadratic}. Moreover, it is a non-trivial model, which makes it possible to understand complex optimization algorithms \cite{arjevani2020tight, cunha2022only, goujaud2022super}. 
The quadratic problem is suitable for observing complex phenomena and providing theoretical insights, which can also be observed in practical scenarios.
 Finally, Appendix \ref{sec:generalization} presents a generalization for smooth functions.

Having said that, we consider a special case of problem \eqref{eq:general_problem} for symmetric matrices $\mL_i$
\begin{equation} \label{eq:het_gen_quad_problem}
      f(x) = \frac{1}{n} \sum \limits_{i=1}^n f_i(x), \qquad f_i(x) \equiv \frac{1}{2} x^\top \mL_i x - x^\top \bb_i.
\end{equation}

In this case, $f(x)$ is $\moL$-smooth, and $\nabla f(x) = \moL x - \ob$, where $\moL = \frac{1}{n} \sum_{i=1}^n \mL_i$ and $\ob \eqdef \frac{1}{n} \sum_{i=1}^n \bb_i$.

\section{Results in the interpolation case}

First, let us examine the case of $\bb_i \equiv 0$, which we call interpolation for quadratics, and perform the analysis for general sketches $\mC_i^k$.
In this case, the gradient estimator \eqref{eq:SGD_generic} takes the form
\begin{equation} \label{eq:het_estimator}
   g^k = \frac{1}{n} \sum \limits_{i=1}^n \mC_i^k \nabla f_i(\mC_i^k x^k) =  \frac{1}{n} \sum \limits_{i=1}^n \mC_i^k \mL_i \mC_i^k x^k = \moB^k x^k
\end{equation}
where $\moB^k \eqdef \frac{1}{n}\sum_{i=1}^n  \mC_i^k \mL_i \mC_i^k$.
We prove the following result for a method with such an estimator.

\begin{theorem} \label{thm:het_quad}
Consider the method \eqref{eq:SGD_generic} with estimator \eqref{eq:het_estimator} for a quadratic problem \eqref{eq:het_gen_quad_problem} with $\moL \succ 0$ and $\bb_i \equiv 0$. Then if $\moW \eqdef \frac{1}{2} \Exp{\moL \moB^k + \moB^k \moL} \succeq 0$ and there exists a constant $\theta > 0$:
\begin{equation} \label{eq:het_exp_sep}
    \Exp{\moB^k \moL \moB^k} \preceq \theta \moW,
\end{equation}
and the step size is chosen as $0 < \gamma \leq \frac{1}{\theta}$, the iterates satisfy
\begin{equation} \label{eq:het_res}
    \frac{1}{K} \sum \limits_{k=0}^{K-1} \Exp{\sqN{\nabla f(x^k)}_{\moL^{-1} \moW \moL^{-1}}}
    \leq 
    \frac{2 \br{f(x^0) - \Exp{f(x^{K})}}}{\gamma K},
\end{equation}
and
\begin{equation} \label{eq:het_res_it}
    \Exp{\sqn{x^k - x^\star}_{\moL}} \leq 
    \br{1 - \gamma \lambda_{\min} \br{\moL^{-\frac{1}{2}}\moW\moL^{-\frac{1}{2}}}}^k \sqn{x^0- x^\star}_{\moL}.
\end{equation}
\end{theorem}

This theorem establishes an $\cO(1/K)$ convergence rate with a constant step size up to a stationary point \eqref{eq:het_res} and linear convergence for the expected distance \eqref{eq:het_res_it} to the optimum $x^\star \eqdef \argmin f(x)$. 
Note that we employ weighted norms in our analysis, as the considered class of loss functions satisfies the matrix $\moL$-smoothness Assumption \ref{ass:matrix-smoothness}.
The use of standard Euclidean distance may result in loose bounds that do not recover correct rates for special cases like gradient descent.

It is important to highlight that the inequality \eqref{eq:het_exp_sep} may not hold (for any $\theta > 0$) in the general case as the matrix $\moW$ is not guaranteed to be positive (semi-)definite in the case of general sampling. 
The intuition behind this issue is that arbitrary sketches $\mC_i^k$ can result in the gradient estimator $g^k$, which is misaligned with the true gradient $\nabla f(x^k)$. Specifically, the inner product $\left\langle\nabla f(x^k), g^k \right\rangle$ can be negative, and there is no expected descent after one step.

Next, we give examples of samplings for which the inequality \eqref{eq:het_exp_sep} can be satisfied.

\textbf{1.} {\bf Identity.} Consider $\mC_i\equiv \mI$. Then $\moB^k = \moL$, $\moB^k\moL\moB^k = \moL^3, \moW = \moL^2 \succ 0$ and hence \eqref{eq:het_exp_sep} is satisfied for $\theta = \lambda_{\max}(\moL)$. So, \eqref{eq:het_res} says that if we choose $\gamma = \nicefrac{1}{\theta}$, then
\[
    \frac{1}{K} \sum \limits_{k=0}^{K-1} \sqN{\nabla f(x^k)}_{\mI }
    \leq 
    \frac{2 \lambda_{\max}(\moL)\br{f(x^0) - f(x^{K})}}{K},
\]
which exactly matches the rate of gradient descent in the non-convex setting.
As for convergence of the iterates, 
the rate in \eqref{eq:het_res_it} is $\nicefrac{\lambda_{\max}(\moL)}{\lambda_{\min}(\moL)}$ which corresponds to the precise gradient descent result for strongly convex functions.

\textbf{2.} {\bf Permutation.} Assume\footnote{This is mainly done to simplify the presentation. Results can be generalized to the case of $n \neq d$ in a similar manner as in \cite{szlendak2022permutation}, which can be found in the Appendix.} $n=d$ and the use of \texttt{Perm-1} (special case of Definition \ref{def:permutation}) sketch $\mC^k_i = n e_{\pi^k_i} e_{\pi^k_i}^\top$, where $\pi^k = (\pi^k_1, \dots, \pi^k_n)$ is a random permutation of $[n]$. Then 
$$\Exp{\moB^k} = \frac{1}{n} \sum \limits_{i=1}^n \Exp{\mC_i^k \mL_i \mC_i^k} = \frac{1}{n} \sum \limits_{i=1}^n n \Diag(\mL_i) = \sum_{i=1}^n \mD_i  = n \moD,$$
where $\moD \eqdef \frac{1}{n} \sum_{i=1}^n \mD_i, \mD_i \eqdef \Diag(\mL_i)$. Then inequality \eqref{eq:het_exp_sep} leads to
\begin{equation*}
   n \moD \moL \moD \preceq \frac{\theta}{2} \br{\moL \moD + \moD \moL},
\end{equation*}
which may not always hold as  $\moL \moD + \moD \moL$ is not guaranteed to be positive-definite---even in the case of $\moL \succ 0$. However, such a condition can be enforced via a slight modification of the permutation sketches, which is done in Section \ref{sec:precond_het_quad}. The limitation of such an approach is that the resulting compressors are no longer unbiased. 

\begin{remark}
Matrix $\moW$ in the case of permutation sketches may not be positive-definite. Consider the following example of a homogeneous ($\mL_i \equiv \mL$) two-dimensional problem:
\begin{equation*}
\mL =
    \left[\begin{array}{cccc}
        a & c \\ 
        c & b
    \end{array}\right].
\end{equation*}
Then 
\begin{equation*}
    \moW = \frac{1}{2} \sbr{\moL \moD + \moD \moL} =
    \left[\begin{array}{cccc}
    a^2 & c(a+b)/2 \\ 
    c(a+b)/2 & b^2
    \end{array}\right],
\end{equation*}
which for $c > \frac{2 a b}{a + b}$ has $\mathrm{det}(\moW) < 0$, and thus $\moW\nsucc 0$ according to Sylvester's criterion.
\end{remark}

Next, we focus on the particular case of \textbf{permutation} sketches, which are the most suitable for model partitioning according to Independent Subnetwork Training (IST). In the rest of this section, we discuss how the condition \eqref{eq:het_exp_sep} can be enforced via a specially designed preconditioning of the problem \eqref{eq:het_gen_quad_problem} or modification of the sketch mechanism \eqref{eq:permutation}.

\subsection{Homogeneous problem preconditioning}
\label{sec:precond_homo_quad}
To start, consider a homogeneous setting $f_i(x) = \frac{1}{2} x^\top \mL x$, so $\mL_i \equiv \mL$.
Now define $\mD = \Diag (\mL)$ -- a diagonal matrix with elements equal to the diagonal of $\mL$. Then, the problem can be converted to
\begin{equation} \label{eq:preconditioned}
    f_i(\mD^{-\frac{1}{2}}x) = \frac{1}{2} \br{\mD^{-\frac{1}{2}}x}^\top \mL \br{\mD^{-\frac{1}{2}} x} = \frac{1}{2} x^\top \underbrace{\br{\mD^{-\frac{1}{2}} \mL \mD^{-\frac{1}{2}}}}_{\mcL} x,
\end{equation}
which is equivalent to the original problem after changing the variables $\tilde{x} \eqdef \mD^{-\frac{1}{2}}x$. Note that $\mD = \Diag (\mL)$ is positive-definite as $\mL \succ 0$, and therefore $\mcL \succ 0$. Moreover, the preconditioned matrix $\mcL$ has all ones on the diagonal: $\Diag(\mcL) = \mI$. If we now combine \eqref{eq:preconditioned} with \texttt{Perm-1} sketches 
\begin{equation*}
   \Exp{\moB^k} = \Exp{\frac{1}{n} \sum_{i=1}^n \bfC_i \mcL \bfC_i} = n \Diag(\mcL) = n \mI.
\end{equation*}
Therefore, inequality \eqref{eq:het_exp_sep} takes the form $\tilde{\mW} = n \mcL \succeq \frac{1}{\theta} n^2 \mcL$, which holds for $\theta \geq n$, and the left-hand side of \eqref{eq:het_res} can be transformed (for an accurate comparison to standard methods) in the following way:
\begin{equation*}
    \sqN{\nabla f(x^k)}_{\mcL^{-1} \tilde{\mW} \mcL^{-1}} 
    \geq n \lambda_{\min} \br{\mcL^{-1}} \sqN{\nabla f(x^k)}_{\mI} =
    n \lambda_{\max} (\mcL) \sqN{\nabla f(x^k)}_{\mI}
\end{equation*}
The resulting convergence guarantee is
\begin{equation*}
    \frac{1}{K} \sum \limits_{k=0}^{K-1} \Exp{\sqN{\nabla f(x^k)}_{\mI}}
    \leq 
    \frac{2 \lambda_{\max} (\mcL) \br{f(x^0) - \Exp{f(x^K)}}}{K},
\end{equation*}
which matches classical gradient descent.

\subsection{Heterogeneous sketch preconditioning} \label{sec:precond_het_quad}
In contrast to the homogeneous case, the heterogeneous problem $f_i(x) = \frac{1}{2} x^\top \mL_i x$ cannot be so easily preconditioned by a simple change of variables $\tilde{x} \eqdef \mD^{-\frac{1}{2}}x$, as every client $i$ has its own matrix $\mL_i$. However, this problem can be fixed via the following modification of \texttt{Perm-1}, which scales the output according to the diagonal elements of the local smoothness matrix $\mL_i$:
\begin{equation} \label{eq:imp_perm}
    \tilde{\bfC}_i \eqdef \sqrt{n/\sbr{\mL_i}_{\pi_i, \pi_i}} e_{\pi_i} e_{\pi_i}^\top.
\end{equation}
In this case,  $\Exp{\tilde{\mC}_i \mL_i \tilde{\mC}_i} = \mI$,  $\Exp{\moB^k} = \mI$, and $\moW = \moL$.
Then inequality \eqref{eq:het_exp_sep} is satisfied for $\theta \geq 1$.
 
If one inputs these results into \eqref{eq:het_res}, such convergence guarantee can be obtained
\begin{equation*}
    \frac{1}{K} \sum \limits_{k=0}^{K-1} \Exp{\sqN{\nabla f(x^k)}_{\mI}}
    \leq 
    \frac{2 \lambda_{\max} (\moL) \br{f(x^0) - \Exp{f(x^K)}}}{K},
\end{equation*}
which matches the gradient descent result as well. Thus, we can conclude that heterogeneity does not bring such a fundamental challenge in this scenario. In addition, the method with \texttt{Perm-1} is significantly better in terms of computational and communication complexity, as it requires calculation of the local gradients with respect to much smaller submodels and transmits only sparse updates.

This construction also shows that for $\gamma = 1/\theta = 1$
\begin{equation*}
    \gamma \lambda_{\min} \br{\moL^{-\frac{1}{2}}\moW\moL^{-\frac{1}{2}}} = \lambda_{\min} \br{\moL^{-\frac{1}{2}} \moL\moL^{-\frac{1}{2}}} = 1,
\end{equation*}
which, after plugging into the bound for the iterates \eqref{eq:het_res_it}, shows that the method basically converges in one iteration. This observation indicates that sketch preconditioning can be extremely efficient, although it uses only the diagonal elements of matrices $\mL_i$.

Now that we understand that the method can perform very well in the special case of $\cb_i \equiv 0$, we can move on to a more complicated situation.

\section{Irreducible bias in the general case}

Now we look at the most general heterogeneous case with different matrices and linear terms 
$f_i(x) \equiv \frac{1}{2} x^\top \mL_i x - x^\top \bb_i.$
In this instance, the gradient estimator \eqref{eq:SGD_generic} takes the form
\begin{equation} \label{eq:het_gen_quad_est}
    g^k = \frac{1}{n} \sum \limits_{i=1}^n \mC_i^k \nabla f_i(\mC_i^k x^k) =  \frac{1}{n} \sum \limits_{i=1}^n \mC_i^k \br{\mL_i \mC_i^k x^k - \mathrm{b}_i} = \moB^k x^k - \oCb,
\end{equation}
where $\oCb = \frac{1}{n}\sum_{i=1}^n \mC_i^k \bb_i$. Herewith let us use a heterogeneous permutation sketch preconditioner \eqref{eq:imp_perm}, as in Section \ref{sec:precond_het_quad}.
Then $\Exp{\moB^k} = \mI$ and $\Exp{\oCb} = \frac{1}{\sqrt{n}} \cDb$, where $\cDb \eqdef \frac{1}{n} \sum_{i=1}^n \mD_i^{-\frac{1}{2}} \bb_i$.
Furthermore, the expected gradient estimator \eqref{eq:het_gen_quad_est} results in $\Exp{g^k} = x^k - \frac{1}{\sqrt{n}} \cDb$ and can be transformed in the following manner:
\begin{equation} \label{eq:het_gen_estimator}
    \Exp{g^k} 
    =
    \moL^{-1} \moL x^k \pm \moL^{-1} \ob - \frac{1}{\sqrt{n}} \cDb 
    = 
    \moL^{-1} \nabla f(x^k) + \underbrace{\moL^{-1}\ob - \frac{1}{\sqrt{n}}\cDb}_{h},
\end{equation}
which reflects the decomposition of the estimator into the optimally preconditioned true gradient and a bias, depending on the linear terms $\bb_i$.

\subsection{Bias of the method} \label{sec:het_method_bias}

Estimator \eqref{eq:het_gen_estimator} can be directly plugged (with proper conditioning) into the general SGD update \eqref{eq:SGD_generic}
\begin{equation} \label{eq:het_iterates}
    \Exp{x^{k+1}} = x^k - \gamma \Exp{g^k} = (1 - \gamma) x^k + \frac{\gamma}{\sqrt{n}} \cDb = \br{1 - \gamma}^{k+1} x^0 + \frac{\gamma}{\sqrt{n}} \cDb \sum \limits_{j=0}^k (1 - \gamma)^j.
\end{equation}

The resulting recursion \eqref{eq:het_iterates} is exact, and its asymptotic limit can be analyzed.
Thus, for constant $\gamma < 1$, by using the formula for the sum of the first $k$ terms of a geometric series, one gets
\begin{eqnarray*}
    \Exp{x^k} =
    \br{1 - \gamma}^k x^0 + \frac{1-(1 - \gamma)^k}{\sqrt{n}} \cDb \underset{k \to \infty}{\longrightarrow} \frac{1}{\sqrt{n}} \cDb,
\end{eqnarray*}
which shows that in the limit, the first initialization term (with $x^0$) vanishes while the second converges to $\frac{1}{\sqrt{n}} \cDb$. 
This reasoning shows that the method does not converge to the exact solution
$$\Exp{x^k} \to x^\infty \neq x^\star 
\in \argmin\limits_{x\in \bR^d} \cbr{\frac{1}{2} x^\top \moL x - x^\top \ob }, $$ 
which for the positive-definite $\moL$ can be defined as $x^\star = \moL^{-1} \ob$, while $x^\infty = \frac{1}{n\sqrt{n}} \sum_{i=1}^n \mD_i^{-\frac{1}{2}} \bb_i$.
So, in general, there is an unavoidable bias. However, in the limit case: $n = d \to \infty$, the bias diminishes.

\subsection{Generic convergence analysis} \label{sec:gen_convergence}
While the analysis in Section \ref{sec:het_method_bias} is precise, it does not allow us to compare the convergence of IST to standard optimization methods. Therefore, we also analyze the non-asymptotic behavior of the method to understand the convergence speed. Our result is formalized in the following theorem:

\begin{theorem} \label{thm:het_gen_gradient}
Consider the method \eqref{eq:SGD_generic} with the estimator \eqref{eq:het_gen_quad_est} for the quadratic problem \eqref{eq:het_gen_quad_problem} with the positive-definite matrix $\moL \succ 0$. Assume that for every $\mD_i \eqdef \Diag(\mL_i)$ matrices $\mD_i^{-\frac{1}{2}}$ exist, scaled permutation sketches \eqref{eq:imp_perm}
are used, and heterogeneity is bounded as $\Exp{\sqN{g^k - \Exp{g^k}}_{\moL}} \leq \sigma^2$. Then, for the step size chosen as follows:
\begin{equation*}
    0 < \gamma \leq \gamma_{c, \beta} \eqdef \frac{1/2 - \beta}{\beta + 1/2},
\end{equation*}
where $\gamma_{c, \beta} \in (0, 1]$ for $\beta \in (0, 1/2)$, the iterates satisfy
\begin{equation} \label{eq:het_gen_grad_res}
    \frac{1}{K} \sum \limits_{k=0}^{K-1} \Exp{\sqN{\nabla f(x^k)}_{\moL^{-1}}}
    \leq 
    \frac{2\br{f(x^0) - \Exp{f(x^K)}}}{\gamma K} + \br{2\beta^{-1}\br{1 - \gamma} + \gamma} \sqn{h}_{\moL} + \gamma \sigma^2,
\end{equation}
where $\moL = \frac{1}{n} \sum_{i=1}^n \mL_i, h = \moL^{-1}\ob - \frac{1}{n^{3/2}} \sum_{i=1}^n \mD_i^{-\frac{1}{2}} \bb_i$ and $\ob = \frac{1}{n} \sum_{i=1}^n \bb_i$.
\end{theorem}

Note that the derived convergence upper bound has a neighborhood proportional to the bias of the gradient estimator $h$ and level of heterogeneity $\sigma^2$. Some of these terms with factor $\gamma$ can be eliminated by decreasing the learning rate (e.g., $\sim 1/\sqrt{K}$). However, such a strategy does not diminish the term with a multiplier $2 \beta^{-1} \br{1 - \gamma}$, making the neighborhood irreducible. Moreover, this term can be eliminated for $\gamma = 1$, which also minimizes the first term that decreases as $1/K$. However, this step size choice maximizes the terms with factor $\gamma$.
Thus, there exists an inherent trade-off between convergence speed and the size of the neighborhood.

In addition, convergence to the stationary point is measured by the weighted $\moL^{-1}$ squared norm of the gradient. At the same time, the neighborhood term depends on the weighted by $\moL$ norm of $h$. This fine-grained decoupling is achieved by carefully applying the Fenchel-Young inequality and provides a tighter characterization of the convergence compared to using standard Euclidean distances.

\paragraph{Homogeneous case.} In this scenario, every worker has access to all data $f_i(x) \equiv \frac{1}{2} x^\top \mL x - x^\top \bb$. Then diagonal preconditioning of the problem can be used, as in the previous Section \ref{sec:precond_homo_quad}. This results in a gradient
$\nabla f(x) = \mcL x - \cb$ for $\mcL = \mD^{-\frac{1}{2}} \mL \mD^{-\frac{1}{2}}$ and $\cb = \mD^{-\frac{1}{2}} \bb$. If this expression is further combined with a permutation sketch (scaled by $1/\sqrt{n}$)
$\mC_i' \eqdef \sqrt{n} e_{\pi_i} e^{\top}_{\pi_i}$, 
the resulting gradient estimator is:
\begin{equation} \label{eq:hom_gen_estimator}
    g^k = x^k - \frac{1}{\sqrt{n}} \cb 
    = 
    \mcL^{-1} \nabla f(x^k) + \tilde{h},
\end{equation}
for $\tilde{h} = \mcL^{-1}\cb - \frac{1}{\sqrt{n}}\cb$.
In this case, the heterogeneity term $\sigma^2$ from the upper bound \eqref{eq:het_gen_grad_res} disappears as $\Exp{\sqN{g^k - \Exp{g^k}}_{\moL}} = 0$, which significantly decreases the neighborhood size. However, the bias term depending on $\tilde{h}$ still remains, as the method does not converge to the exact solution 
$x^k \to x^\infty \neq x^\star = \mcL^{-1} \cb$ for positive-definite $\mcL$. Nevertheless the method's fixed point $x^\infty = \cb / \sqrt{n}$ and solution $x^\star$ can coincide when $\mcL^{-1} \cb = \frac{1}{\sqrt{n}} \cb$, which means that $\cb$ is the right eigenvector of matrix $\mcL^{-1}$ with eigenvalue $\frac{1}{\sqrt{n}}$.

Let us contrast the obtained result \eqref{eq:het_gen_grad_res} with the non-convex rate of SGD \cite{khaled2022better} with constant step size $\gamma$ for $L$-smooth and lower-bounded $f$
\begin{equation} \label{eq:SGD_rate}
    \min \limits_{0 \leq k \leq K-1} \Exp{\sqN{\nabla f(x^k)}}
    \leq 
    \frac{6\br{f(x^0) - \inf f}}{\gamma K} + \gamma L C,
\end{equation}
where constant $C$ depends, for example, on the variance of the stochastic gradient estimator.
Observe that the first term in the compared upper bounds \eqref{eq:SGD_rate} and \eqref{eq:het_gen_grad_res} is almost identical and decreases with speed $1/K$. However, unlike \eqref{eq:het_gen_grad_res}, the neighborhood for SGD can be completely eliminated by reducing the step size $\gamma$. This highlights a fundamental difference between our results and unbiased methods.
The intuition behind this issue is that for SGD-type methods like compressed gradient descent
\begin{equation*}
    x^{k+1} = x^k - \gamma \cC(\nabla f(x^k))
\end{equation*}
the gradient estimate is unbiased and enjoys the property that variance 
\begin{equation*}
    \Exp{ \sqn{\cC(\nabla f(x^k)) - \nabla f(x^k)}} \leq \omega \sqn{\nabla f(x^k)}
\end{equation*}
goes down to zero as the method progresses because $\nabla f(x^k) \to \nabla f(x^\star) = 0$ in the unconstrained case.
In addition, any stationary point $x^\star$ ceases to be a fixed point of the iterative procedure as
\begin{equation*}
    x^\star \neq x^\star - \gamma \nabla f(\cC(x^\star)),
\end{equation*}
in the general case, unlike for compressed gradient descent
with both biased and unbiased compressors $\cC$. Thus, even if the method---computing the gradient with a sparse model---is initialized from the \emph{solution} after one gradient step, the method may get away from the optimum.

\subsection{Comparison to previous works}

\paragraph{Independent Subnetwork Training \cite{yuan2022distributed}.} There are several improvements over the previous works that tried to theoretically analyze the convergence of distributed IST. 

The first difference is that our results allow for an almost arbitrary level of model sparsification, i.e., will work for any $\omega \geq 0$ as permutation sketches can be viewed as a special case of compression operators \eqref{def:unbiased_compressor}. This represents a significant improvement over the work of \citet{yuan2022distributed}, which demands\footnote{$\mu$ refers to a constant from the Polyak-Łojasiewicz (or strong convexity) condition. In case of a quadratic problem with positive-definite matrix $\mA$  constant $\mu$ equals to $\lambda_{\min} (\mA)$} $\omega \lesssim \nicefrac{\mu^2}{L^2}$. Such a requirement is very restrictive as the condition number $L/\mu$ of the loss function $f$ is typically very large for any non-trivial optimization problem. Thus, the sparsifier's \eqref{eq:rand-k} variance $\omega = d/q-1$ has to be very close to 0 and $q \approx d$. Thus, the previous theory allows almost no compression (sparsification) because it is based on the analysis of gradient descent with compressed iterates \cite{khaled2019gradient}.

The second distinction is that the original IST work \cite{yuan2022distributed} considered a single node setting, and thus their convergence bounds did not capture the effect of heterogeneity, which we believe is of crucial importance for distributed settings \cite{chraibi2019distributed, shulgin2022shifted}.
Moreover, the original work considers the Lipschitz continuity of the loss function $f$, which is not satisfied for a simple quadratic model. A more detailed comparison, including additional assumptions on the gradient estimator made by \citet{yuan2022distributed}, is presented in the Appendix \ref{sec:app_comparison}.

\paragraph{FL with Model Pruning.}
In a recent work, \citet{zhou2022convergence} made an attempt to analyze a variant of the FedAvg algorithm with sparse local initialization and compressed gradient training (pruned local models). They considered a case of $L$-smooth loss and a sparsification operator satisfying a similar condition to \eqref{def:unbiased_compressor}. However, they also assumed that the squared norm of the stochastic gradient is uniformly bounded \eqref{eq:bounded_gradient}, which is \say{pathological} \cite{khaled2020tighter}---especially in the case of local methods---as it does not allow the analysis to capture the very important effect of heterogeneity and can result in vacuous bounds. 

In the Appendix \ref{sec:app_comparison}, we present some limitations of other relevant previous approaches to training with compressed models: too restrictive assumptions on the algorithm \cite{mohtashami2022masked} or non-applicability in our problem setting \cite{chayti2022optimization}. In addition, we discuss the differences between IST and 3D parallelism \cite{shoeybi2019megatron}.

\section{Experiments} \label{sec:experiments}

To empirically validate our theoretical framework and its implications, we focus on a carefully controlled setting that satisfies the assumptions of our work. Specifically, we consider a quadratic problem defined in \eqref{eq:het_gen_quad_problem}. As a reminder, the local loss function is defined as $$f_i(x) = \frac{1}{2} x^\top \mL_i x - x^\top \bb_i,$$ where $\mL_i = \mB_i^\top \mB_i$. Entries of the matrices $\mB_i \in \bR^{d \times d}$, vectors $\bb_i \in \Rd$, and initialization $x^0\in\Rd$ are generated from a standard Gaussian distribution $\mathcal{N}(0, 1)$.

\begin{figure}
\centering
\begin{subfigure}{0.49\textwidth}
    \includegraphics[width=\linewidth]{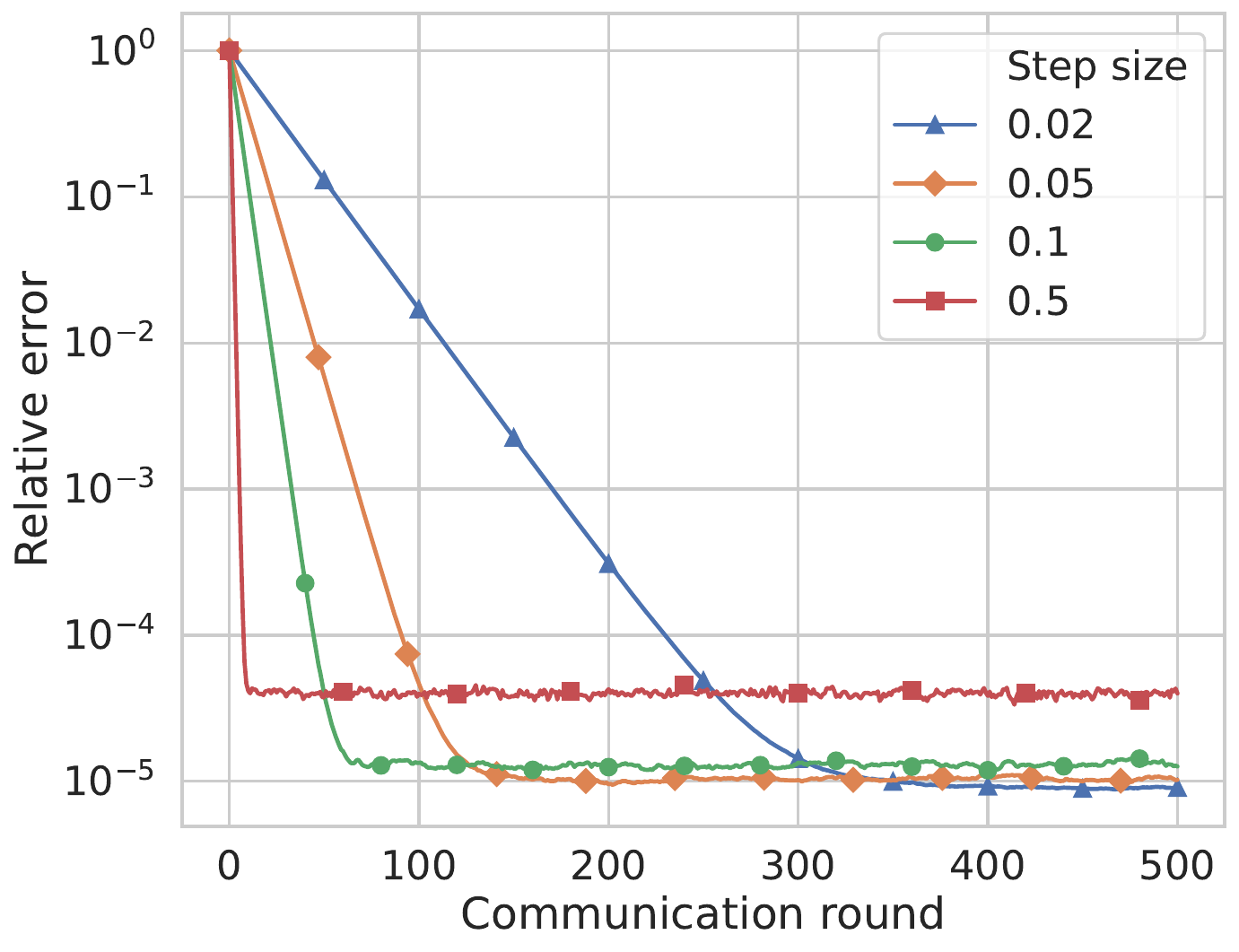}
    \caption{Function convergence for heterogeneous case.}
    \label{fig:hetero_convergence}
\end{subfigure}
\begin{subfigure}{0.5\textwidth}
    \includegraphics[width=\linewidth]{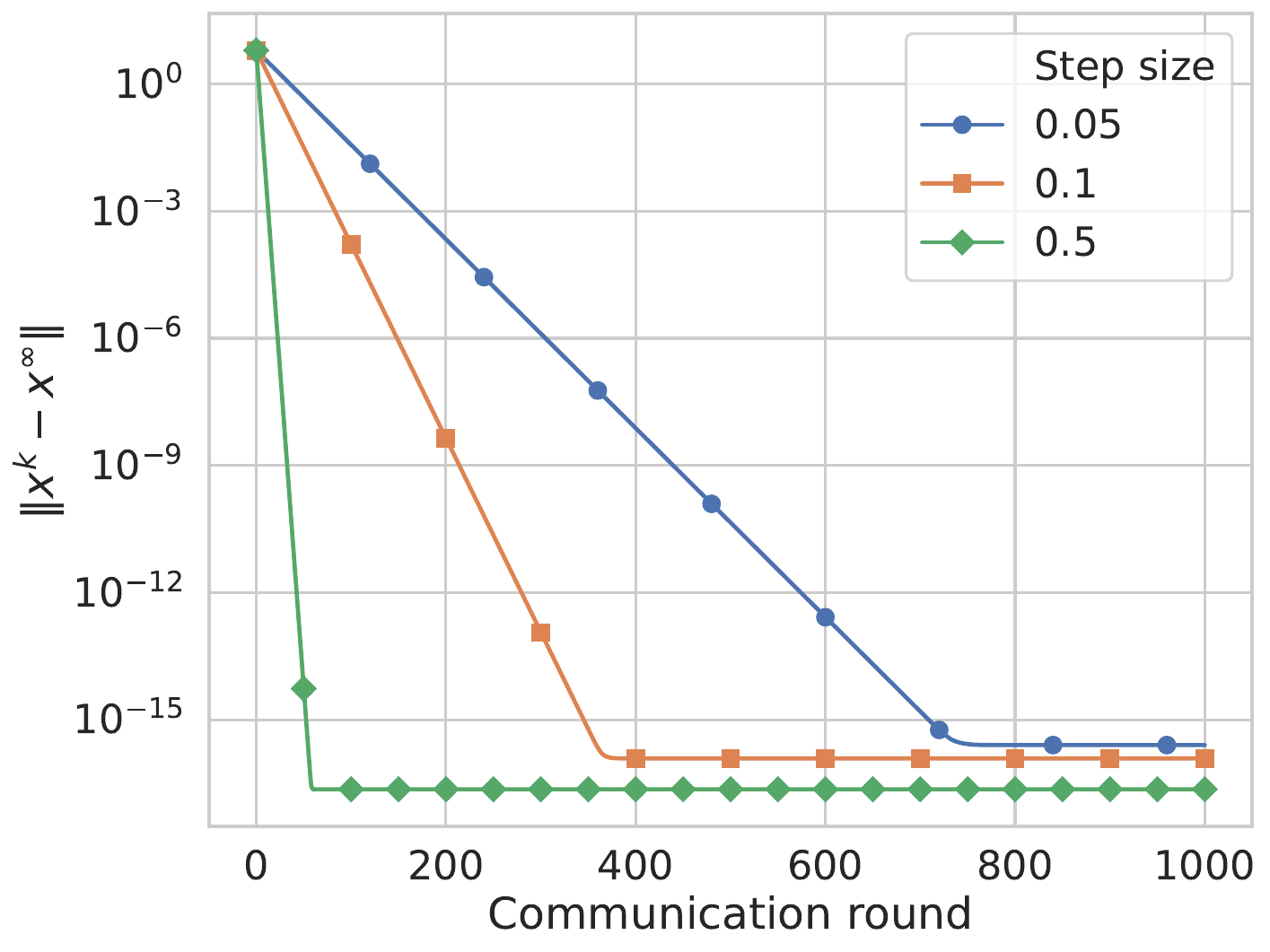}
    \caption{Iterates convergence for homogeneous case.}
    \label{fig:homo_stationary}
\end{subfigure}
\caption{Performance of simplified IST on quadratic problem for varying step size values.}
\label{fig:varying_ss}
\end{figure}

\textbf{Heterogeneous setting.} In Figure \ref{fig:hetero_convergence}, we present the performance of the simplified Independent Subnetwork Training (IST) algorithm (update \eqref{eq:SGD_generic} with estimator \eqref{eq:het_gen_quad_est}) for a heterogeneous problem. We fix the dimension $d$ to 1000 and the number of computing nodes $n$ to 10. We evaluate the logarithm of a relative functional error $\log (f(x^k) - f(x^\star)) / (f(x^0) - f(x^\star))$, while the horizontal axis denotes the number of communication rounds required to achieve a certain error tolerance. According to our theory \eqref{eq:gen_func_res}, the method converges to a neighborhood of the solution, which depends on the chosen step size. Specifically, a larger step size allows for faster convergence but results in a larger neighborhood.

\textbf{Homogeneous setting.} In Figure \ref{fig:homo_stationary}, we demonstrate the convergence of the iterates $x^k$ for a homogeneous problem with $d = n = 50$.
The results are in close agreement with our theoretical predictions for the estimator \eqref{eq:hom_gen_estimator}. We observe that the distance to the method's expected fixed point $x^\infty = \cb / \sqrt{n}$ decreases linearly for different step size values. This confirms that IST may not converge to the optimal solution $x^\star = \mcL^{-1} \cb$ of the original problem \eqref{eq:het_gen_quad_problem} in general (no interpolation) cases. In addition, there are no visible oscillations in comparison to the heterogeneous case.

\subsection{Neural network results}
\begin{figure}
\centering
\begin{subfigure}{0.485\textwidth}
    \includegraphics[width=\linewidth]{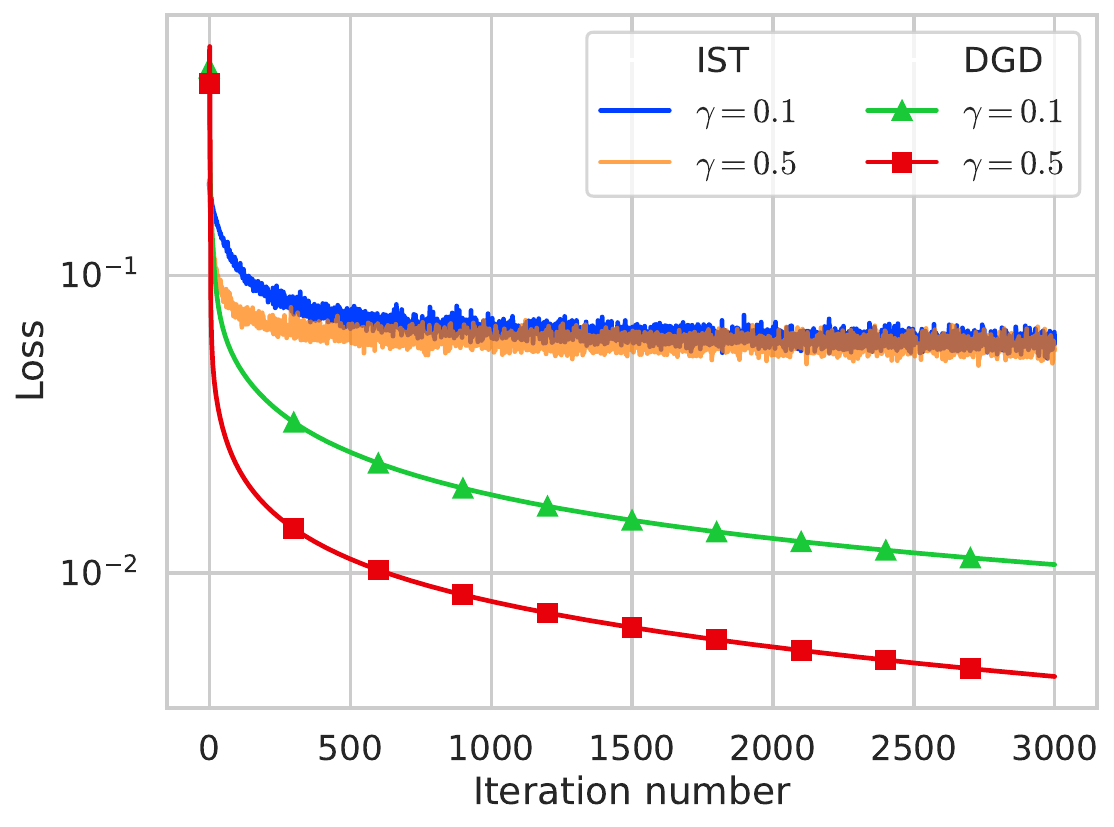}%
    \caption{Comparison of IST and Distributed GD (DGD).}
    \label{fig:ist_vs_dgd}%
\end{subfigure}%
\begin{subfigure}{0.51\textwidth}
    \includegraphics[width=\linewidth]{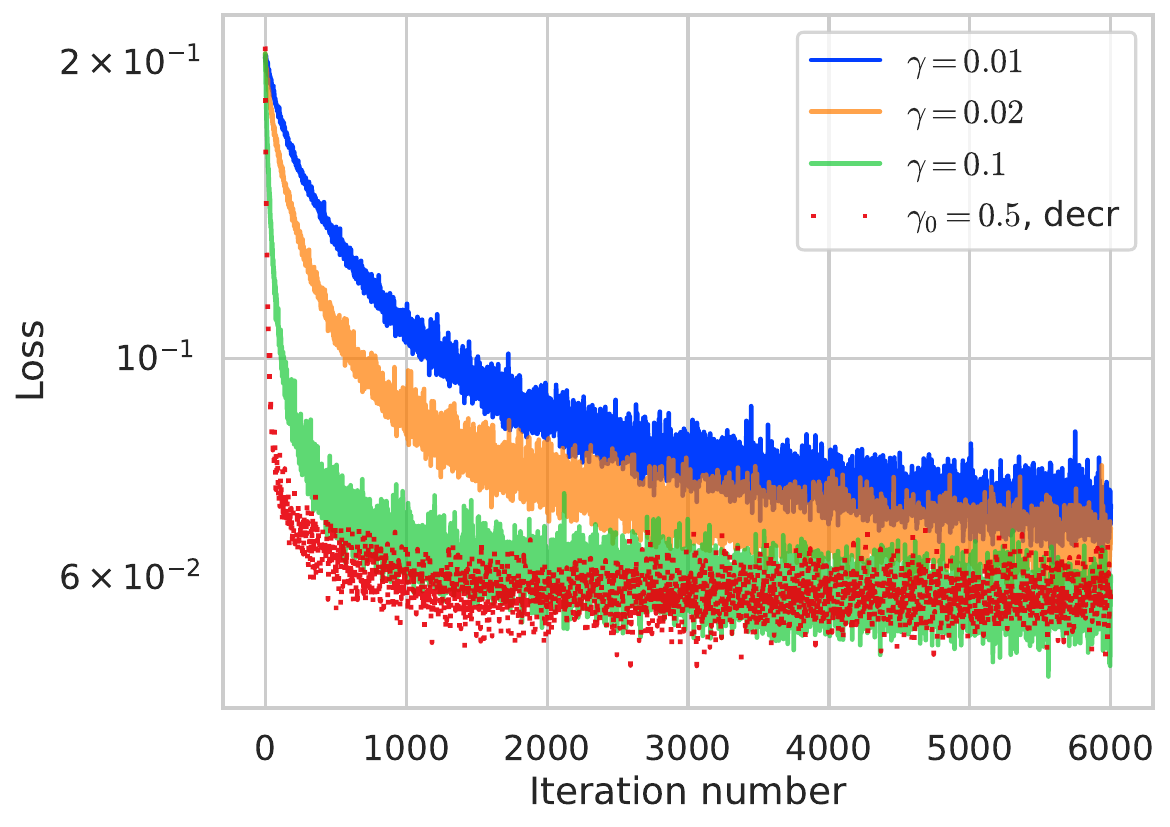}%
    \caption{Performance of IST with different step sizes.}
    \label{fig:ist_decr}%
\end{subfigure}
\caption{Experimental study of IST on a neural network problem.}
\end{figure}

We closely follow the experimental setup of \citet{liao2022on} and use a one-hidden-layer neural network with ReLU activations. For completeness, we repeat some of the details here.
ResNet-50 model \cite{he2016deep} pre-trained on ImageNet is used as a feature extractor and concatenated with two fully connected layers. The resulting model is then trained on the CIFAR-10 \cite{krizhevsky2009cifar} dataset. We take the outputs of the re-trained ResNet-50 as the input features resulting in $d = 2048$, and the logit outputs of the combined model are used as the labels. 

The goal of the experiment is to study the optimization performance of the IST method described in Section \ref{sec:setup}. Namely, we consider Algorithm \ref{alg:IST} with 1) $\mC_i$ chosen as \texttt{Perm-q} \eqref{eq:permutation} for IST and 2) $\mC_i = \mI$ for Distributed Gradient Descent (DGD). Both methods are implemented across $n=10$ nodes, employing constant step sizes $\gamma$, and one local step per communication round.

Figure \ref{fig:ist_vs_dgd} shows the logarithm of the average loss of submodels after every iteration. 
The main observation is that IST's performance stagnates (starts oscillating) at some point, unlike DGD's. Namely, the studied method converges to the neighborhood whose size is basically the same for different step size values. This phenomenon distinguishes the method from standard SGD.
Moreover, larger $\gamma$ speeds up DGD and allows reaching the error floor (oscillation level) faster for IST. These observations agree well with our Theorems \ref{thm:het_gen_gradient} and \ref{thm:gen_function} as the convergence upper bound \eqref{eq:het_gen_grad_res} has an irreducible term proportional to the bias norm $\sqN{h}_{\moL}$ of the gradient estimator. This term can not be eliminated completely by decreasing the step size $\gamma$, unlike SGD \eqref{eq:SGD_rate}.

In Figure \ref{fig:ist_decr}, we take a closer look at the training loss during IST optimization. For this problem, the situation differs from Figure \ref{fig:hetero_convergence} as for even smaller step size values ($\gamma \in \{0.01, 0.02\}$), the method converges to a higher error floor. Interestingly, if $\gamma$ is decreased by 10 every 1000 iterations, the method's performance ({\color{red} red} dotted curve) almost does not change. This can be explained by the second term $\br{1-\gamma}\beta^{-1} \sqN{h}_{\moL}$ from Theorem \ref{thm:het_gen_gradient} convergence bound \eqref{eq:het_gen_grad_res}, which increases for smaller $\gamma$.
The observed effect distinguishes IST from the typical training situation in deep learning.

\section{Conclusions and Future Work}

In this study, we introduced a novel approach to understanding training with combined model and data parallelism for a quadratic model. Our framework sheds light on distributed submodel optimization, which reveals the advantages and limitations of Independent Subnetwork Training (IST). Moreover, we accurately characterized the behavior of the considered method in both homogeneous and heterogeneous scenarios without imposing restrictive assumptions on the gradient estimators.

In future research, it would be valuable to explore extensions of our findings to settings that are closer to scenarios, such as cross-device federated learning. This could involve investigating partial participation support, leveraging local training benefits, and ensuring robustness against stragglers. 

It would be interesting to generalize our analysis to non-quadratic scenarios without relying on pathological assumptions. 
This work shows a somewhat negative result regarding worst-case IST performance for standard Empirical Risk Minimization problem \eqref{eq:general_problem}. However, IST was recently shown to be effective at solving an alternative optimization problem formulation by \citet{demidovich2023mast}.
Another potential promising research direction is algorithmic modifications of the original IST to solve the fundamental problems highlighted in this work and acceleration of training.

\section*{Acknowledgments}
We would like to thank anonymous reviewers, Avetik Karagulyan (KAUST) and Andrea Devlin (KAUST) for their helpful comments and suggestions to improve the  manuscript.

\bibliography{references}
\bibliographystyle{plainnat}

\newpage
\appendix


\tableofcontents

\newpage

\begin{figure*}[h]
\centering
\includegraphics[width=0.7\linewidth]{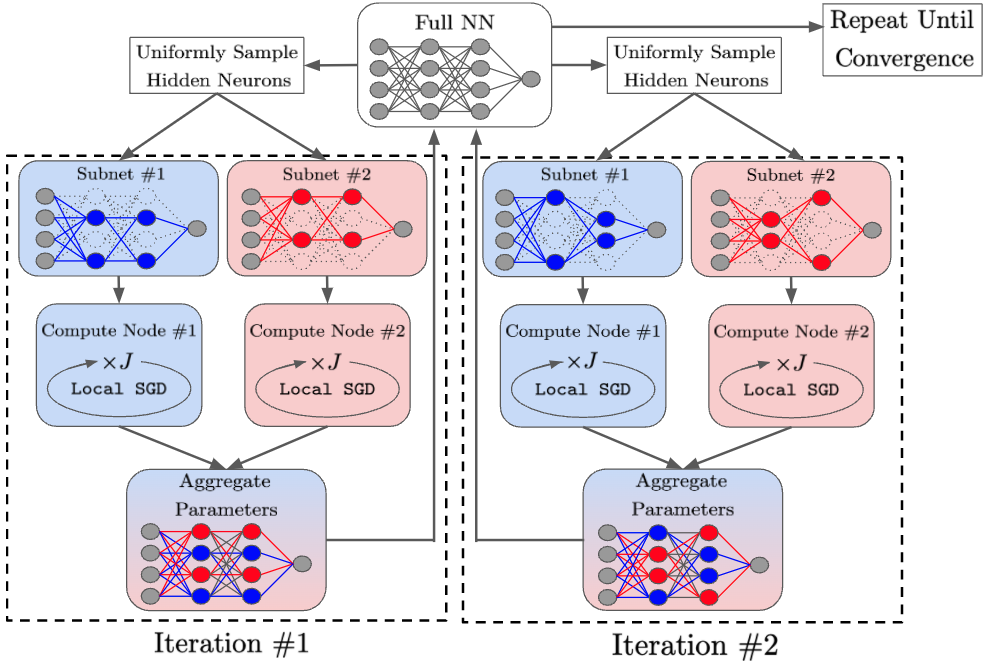}%
\caption{Schematic depiction of a Neural Network trained with IST across two nodes. Source: \cite{yuan2022distributed}.}
\label{fig:IST_schematic}
\end{figure*}

\section{Basic and auxiliary facts}

$\mL$-matrix smoothness:
\begin{equation} \label{eq:L-matrix-smooth_}
   f(x + h) \leq f(x) + \left\langle\nabla f(x), h \right\rangle + \frac{1}{2} \left\langle \mL h, h \right\rangle, \qquad \forall x, h \in \Rd.
\end{equation}

\paragraph{Basic Inequalities.}
For all vectors $a, b\in \Rd$ and random vector $X\in \Rd$:
\begin{equation} \label{eq:dot_product}
    2\<a, b> = \sqn{a} + \sqn{b} - \sqn{a - b},
\end{equation}
\begin{equation} \label{eq:triangle}
    \sqn{a + b} \leq (1 + \beta) \sqn{a} + (1 + \beta^{-1}) \sqn{b}, \text{ for } \beta > 0
\end{equation}
\begin{equation} \label{eq:variance_decomposition}
    \E\sqn{X - a} = \E\sqn{X - \E X} + \sqn{\E X - a}.
\end{equation}
For a set of $n \geq 1$ vectors $a_{1}, \ldots, a_{n} \in \bR^d$ it holds
\begin{equation} \label{eq:average_inequality}
    \sqN{\frac{1}{n} \sum_{i=1}^{n} a_i} \leq \frac{1}{n} \sum_{i=1}^{n} \sqn{a_i}.
\end{equation}

\begin{lemma}[Fenchel–Young inequality] \label{lemma:fenchel_young}
For any function $f$ and its convex conjugate $f^*$, Fenchel's inequality (also known as the Fenchel–Young inequality) holds for every $x, y \in \bR^d$
\begin{equation*}
    \left\langle x, y \right\rangle \leq f(x) + f^*(y).
\end{equation*}
The proof follows from the definition of conjugate: $f^*(y) \eqdef \sup_{x'} \cbr{\left\langle y, x' \right\rangle - f(x')} \geq \left\langle y, x \right\rangle - f(x)$.
\end{lemma}
In the case of a quadratic function $f(x) = \beta \sqn{x}_{\mL}$, we can compute $f^*(y) = \frac{1}{4} \beta^{-1} \sqn{y}_{\mL^{-1}}$. Thus 
\begin{equation} \label{eq:fy-ineq_quad}
    \left\langle x, y \right\rangle \leq \beta \sqn{x}_{\mL} + \frac{1}{4} \beta^{-1} \sqn{y}_{\mL^{-1}}.
\end{equation}

\section{Proofs}

\subsection{Permutation sketch computations}

All derivations in this section are performed for the $n = d$ case.

\paragraph{Classical Permutation Sketching.}
\texttt{Perm-1}: $\mC_i = n e_{\pi_i} e_{\pi_i}^\top$, where $\pi = (\pi_1, \dots, \pi_n)$ is a random permutation of $[n]$.
For the homogeneous problem $\mL_i \equiv \mL$:

\begin{equation*}
    \Exp{\moB^k} = \Exp{\frac{1}{n} \sum_{i=1}^n \bfC_i \mL \bfC_i} = n \Diag(\mL) 
\end{equation*}
Then 
\begin{equation*}
    2 \moW = \Exp{\mL \moB^k + \moB^k \mL} = n \br{\mL\Diag(\mL) + \Diag(\mL) \mL}
\end{equation*}
and
\begin{equation*}
    \Exp{\moB^k\mL\moB^k} = n^2 \Diag(\mL) \mL \Diag(\mL).
\end{equation*}
By repeating basically the same calculations for $\mC'_i = \sqrt{n} e_{\pi_i} e_{\pi_i}^\top$ we have that
\begin{equation*}
    \Exp{\moB^k} = \Exp{\frac{1}{n} \sum_{i=1}^n \mC'_i \mL \mC'_i} = \Diag(\mL),
\end{equation*}
and 
$\Exp{\moB^k\mL\moB^k} = \Diag(\mL) \mL \Diag(\mL)$, $2 \moW = \Exp{\mL \moB^k + \moB^k \mL} = \mL\Diag(\mL) + \Diag(\mL) \mL$.

\subsubsection{Heterogeneous sketch preconditioning.} \label{sec:het_sketch_precond}
We recall the following modification of \texttt{Perm-1}:
\begin{equation*}
    \tilde{\bfC}_i \eqdef \sqrt{n/\sbr{\mL_i}_{\pi_i, \pi_i}} e_{\pi_i} e_{\pi_i}^\top.
\end{equation*}
Then 
\begin{equation*}
    \Exp{\tilde{\mC}_i \mL_i \tilde{\mC}_i} = 
    \Exp{n [\mL_i]^{-1}_{\pi_i, \pi_i} e_{\pi_i} e_{\pi_i}^\top \mL_i e_{\pi_i} e_{\pi_i}^\top} = 
    \frac{1}{n} \sum_{j=1}^n n e_{j} \mathbf{I}_{j, j} e_{j}^\top = 
    \mI.
\end{equation*}
and 
\begin{eqnarray*}
    \Exp{\moB^k} &=&
    \Exp{\frac{1}{n} \sum_{i=1}^n \tilde{\mC}_i \mL_i \tilde{\mC}_i} 
    \\&=& 
    \frac{1}{n} \sum_{i=1}^n \Exp{n[\mL_i]^{-1}_{\pi_i, \pi_i} e_{\pi_i} e_{\pi_i}^\top \mL_i e_{\pi_i} e_{\pi_i}^\top}
    \\&=& 
    \frac{1}{n} \sum_{i=1}^n \frac{1}{n} \sum_{j=1}^n n[\mL_i]^{-1}_{j, j} e_{j} [\mL_i]_{j,j} e_{j}^\top
    \\&=& 
    \frac{1}{n} \sum_{i=1}^n \sum_{j=1}^n  e_{j} e_{j}^\top
    \\&=&
    \mI.
\end{eqnarray*}
Thus $\moW = \frac{1}{2} \Exp{\moL \moB^k + \moB^k \moL} = \moL$.
On the left hand side of inequality \eqref{eq:het_exp_sep}, we have 
\begin{eqnarray*}
    \Exp{\moB^k \moL \moB^k} 
    &=&
    \Exp{\frac{1}{n} \sum_{i=1}^n \tilde{\mC}_i \mL_i \tilde{\mC}_i \moL \frac{1}{n} \sum_{i=j}^n \tilde{\mC}_j \mL_j \tilde{\mC}_j} 
    \\&=& 
    \frac{1}{n^2} \sum_{i,j=1}^n \Exp{\tilde{\mC}_i \mL_i \tilde{\mC}_i \moL \tilde{\mC}_j \mL_j \tilde{\mC}_j} 
    \\&=& 
    \sum_{i,j=1}^n e_i e_i^\top \moL e_j e_j^\top
    \\&=& 
    \mI \moL \mI
    \\&=&
    \moL.
\end{eqnarray*}

\subsection{Interpolation case: proof of Theorem \ref{thm:het_quad}}

In the quadratic interpolation regime, the linear term is zero
$f_i(x) = \frac{1}{2} x^\top \mL_i x$,
and the gradient estimator has the form 
\begin{equation*}
   g^k = \frac{1}{n} \sum \limits_{i=1}^n \mC_i^k \nabla f_i(\mC_i^k x^k) =  \frac{1}{n} \sum \limits_{i=1}^n \mC_i^k \mL_i \mC_i^k x^k = \moB^k x^k.
\end{equation*}

\begin{proof}
First, we prove the \textbf{stationary point} convergence result \eqref{eq:het_res}.

Using the $\moL$-smoothness of function $f$, we get
\begin{eqnarray*}
    f(x^{k+1}) \overset{\eqref{eq:SGD_generic} }{=} f(x^k - \gamma g^k) 
    & \overset{\eqref{eq:L-matrix-smooth}}{\leq} &
    f(x^k) - \left\langle\nabla f(x^k), \gamma g^k \right\rangle + \frac{\gamma^2}{2} \sqN{g^k}_{\moL} \\
    &\overset{\eqref{eq:het_estimator}}{=} &
    f(x^k) - \gamma \left\langle \moL x^k, \moB^k x^k \right\rangle + \frac{\gamma^2}{2} \sqN{\moB^k x^k}_{\moL} \\
    &= & f(x^k) - \gamma (x^k)^\top \moL \moB^k x^k  + \frac{\gamma^2}{2} 
 (x^k)^\top \moB^k \moL \moB^k x^k.
\end{eqnarray*}
After applying conditional expectation, using its linearity, and the fact that 
$$x^\top \mA x = \frac{1}{2} x^\top \br{\mA + \mA^\top}x$$ we get
\begin{eqnarray*} 
    \Exp{ f(x^{k+1}) \;|\; x^k}
    & \leq & f(x^k) - \gamma (x^k)^\top \Exp{\moL\moB^k} x^k  + \frac{\gamma^2}{2} 
    (x^k)^\top \Exp{\moB^k \moL \moB^k } x^k\\
    &=&f(x^k) - \gamma (x^k)^\top \moW x^k  + \frac{\gamma^2}{2} 
    (x^k)^\top \Exp{\moB^k \moL \moB^k } x^k\\
    &=&f(x^k) - \gamma ( \nabla f(x^k))^\top  \moL^{-1}\moW \moL^{-1} \nabla f(x^k)  
    \\&& \qquad + \frac{\gamma^2}{2} 
    (\nabla f(x^k))^\top \moL^{-1}\Exp{\moB^k \moL \moB^k} \moL^{-1}  \nabla f(x^k) \\
    &\overset{\eqref{eq:het_exp_sep}}{\leq}&f(x^k) - \gamma \norm{\nabla f(x^k)}^2_{\moL^{-1}\moW \moL^{-1}}  + \frac{\theta \gamma^2}{2} \norm{\nabla f(x^k)}^2_{\moL^{-1}\moW \moL^{-1}} \\
    &=&
    f(x^k) - \gamma \br{1 - \nicefrac{\theta\gamma}{2}} \norm{\nabla f(x^k)}^2_{\moL^{-1}\moW \moL^{-1}}
    \\ & \leq &
    f(x^k) - \frac{\gamma}{2} \sqn{\nabla f(x^k)}_{\moL^{-1}
    \moW \moL^{-1}},
\end{eqnarray*} 
where the last inequality holds for the stepsize $\gamma \leq \frac{1}{\theta}$.

Rearranging gives
\begin{equation*}
    \sqN{\nabla f(x^k)}_{\moL^{-1} \moW \moL^{-1}} \leq 
    \frac{2}{\gamma} \br{f(x^k) - \Exp{ f(x^{k+1}) \;|\; x^k}},
\end{equation*}
which after averaging gives the desired result
\begin{equation*}
    \frac{1}{K} \sum_{k=0}^{K-1} \Exp{\sqN{\nabla f(x^k)}_{\moL^{-1} \moW \moL^{-1}}}
    \leq 
    \frac{2}{\gamma K} \sum_{k=0}^{K-1} (f(x^{k}) - \Exp{f(x^{k+1})}) 
    =
    \frac{2 \br{f(x^0) - \Exp{f(x^K)}}}{\gamma K}.
\end{equation*}

Now we show the result for the \textbf{iterates convergence} \eqref{eq:het_res_it}.

Expectation conditioned on $x^k$:
\begin{eqnarray*} 
\Exp{\sqn{x^{k+1} - x^\star}_{\moL}} &=& 
\Exp{\sqn{x^k - \gamma g^k - x^\star}_{\moL}}
\\&=& 
\sqn{x^k - x^\star}_{\moL} - 2 \gamma \left\langle x^k - x^\star, \Exp{\moL\moB^k} (x^k - x^\star) \right\rangle 
\\&&\qquad + \gamma^2 \left\langle \Exp{\moB^k \moL \moB^k}(x^k - x^\star), x^k - x^\star \right\rangle
\\ & \overset{x^\star=0}{=} &
\sqn{x^k - x^\star}_{\moL} - 2 \gamma \left\langle x^k - x^\star, \moW (x^k - x^\star) \right\rangle
\\&&\qquad  + \gamma^2 \left\langle x^k - x^\star, \Exp{\moB^k \moL \moB^k} (x^k - x^\star) \right\rangle
\\ & \overset{\eqref{eq:het_exp_sep}}{\leq} &
\sqn{x^k - x^\star}_{\moL} - 2 \gamma \left\langle x^k - x^\star, \moW (x^k - x^\star) \right\rangle + \theta\gamma^2 \left\langle x^k - x^\star, \moW (x^k - x^\star) \right\rangle
\\ & = &
\sqn{x^k - x^\star}_{\moL} - 2\gamma \br{1 - \theta\gamma/2} \sqN{\moL^{\frac{1}{2}} (x^k - x^\star)}_{\moL^{-\frac{1}{2}}\moW\moL^{-\frac{1}{2}}}
\\ & \overset{\gamma \leq \nicefrac{1}{\theta}}{\leq} &
\sqn{x^k - x^\star}_{\moL} - \gamma \sqN{\moL^{\frac{1}{2}}(x^k - x^\star)}_{\moL^{-\frac{1}{2}}\moW\moL^{-\frac{1}{2}}}
\\ & \leq &
\sqn{x^k - x^\star}_{\moL} - \gamma \lambda_{\min} \br{\moL^{-\frac{1}{2}}\moW\moL^{-\frac{1}{2}}} \sqN{\moL^{\frac{1}{2}}(x^k - x^\star)}
\\ & = &
\br{1 - \gamma \lambda_{\min} \br{\moL^{-\frac{1}{2}}\moW\moL^{-\frac{1}{2}}}} \sqn{x^k - x^\star}_{\moL}.
\end{eqnarray*} 
After unrolling the recursion we obtain the convergence result
\begin{equation*}
    \Exp{\sqn{x^{k+1} - x^\star}_{\moL}} \leq 
    \br{1 - \gamma \lambda_{\min} \br{\moL^{-\frac{1}{2}}\moW\moL^{-\frac{1}{2}}}}^{k+1} \sqn{x^0 - x^\star}_{\moL}.
\end{equation*}

\end{proof}

\subsection{Non-zero solution}

As a reminder, in the most general case, the problem has the form
\begin{equation*}
   f(x) = \frac{1}{n} \sum \limits_{i=1}^n f_i(x), \qquad f_i(x) \equiv \frac{1}{2} x^\top \mL_i x - x^\top \bb_i.
\end{equation*}

with the gradient estimator
\begin{equation*}
   g^k = \frac{1}{n} \sum \limits_{i=1}^n \mC_i^k \nabla f_i(\mC_i^k x^k) =  \frac{1}{n} \sum \limits_{i=1}^n \mC_i^k \br{\mL_i \mC_i^k x^k - \mathrm{b}_i} = \moB^k x^k - \frac{1}{n} \sum \limits_{i=1}^n \mC_i^k \bb_i.
\end{equation*}

\paragraph{General calculations for estimator \eqref{eq:het_gen_quad_est}.}
In the heterogeneous case, the following sketch preconditioner is used
\begin{equation*}
    \tilde{\bfC}_i \eqdef \sqrt{n/\sbr{\mL_i}_{\pi_i, \pi_i}} e_{\pi_i} e_{\pi_i}^\top.
\end{equation*}
Then $\Exp{\moB^k} = \mI$ (calculation was done as in Section \ref{sec:het_sketch_precond}) and
\begin{eqnarray*}
    \Exp{\oCb} 
    &=& 
    \frac{1}{n}\sum_{i=1}^n \Exp{\tilde{\mC}_i^k \bb_i} 
    \\&=& 
    \frac{1}{n}\sum_{i=1}^n \Exp{\sqrt{n} [\mL_i]^{-\frac{1}{2}}_{\pi_i, \pi_i} e_{\pi_i} e_{\pi_i}^\top \bb_i} \\&=&
    \frac{1}{n} \sum_{i=1}^n \frac{1}{n} \sum_{j=1}^n \sqrt{n} [\mL_i]^{-\frac{1}{2}}_{j,j} e_{j} [\bb_i]_j 
    \\&=&
    \frac{1}{n} \sum_{i=1}^n \frac{1}{n} \sqrt{n} \mD_i^{-\frac{1}{2}} \bb_i
    \\&=&
    \frac{1}{\sqrt{n}} \frac{1}{n} \sum_{i=1}^n \mD_i^{-\frac{1}{2}} \bb_i
    \\&=&
    \frac{1}{\sqrt{n}} \underbrace{\oDb}_{\cDb}
\end{eqnarray*}

\subsubsection{Generic convergence analysis for heterogeneous case: proof of Theorem \ref{thm:het_gen_gradient}.}

Here we formulate and further prove a more general version of Theorem \ref{thm:het_gen_gradient}, which is obtained as a special case of the next result for $c=1/2$.

\begin{theorem}
Consider the method \eqref{eq:SGD_generic} with estimator \eqref{eq:het_gen_quad_est} for a quadratic problem \eqref{eq:het_gen_quad_problem} with positive-definite matrix $\moL \succ 0$. Then, if for every $\mD_i \eqdef \Diag(\mL_i)$ matrices $\mD_i^{-\frac{1}{2}}$ exist, scaled permutation sketches $\bfC_i  \eqdef \sqrt{n} [\mL_i^{-\frac{1}{2}}]_{\pi_i, \pi_i} e_{\pi_i} e_{\pi_i}^\top$ are used and heterogeneity is bounded as $\Exp{\sqN{g^k - \Exp{g^k}}_{\moL}} \leq \sigma^2$. Then, the step size is chosen as
\begin{equation*}
    0 < \gamma \leq \gamma_{c, \beta} \eqdef \frac{1 - c - \beta}{\beta + 1/2},
\end{equation*}
where $\gamma_{c, \beta} \in (0, 1]$ for $\beta + c < 1$, the iterates satisfy
\begin{equation*}
    \frac{1}{K} \sum_{k=0}^{K-1} \Exp{\sqN{\nabla f(x^k)}_{\moL^{-1}}}
    \leq 
    \frac{f(x^0) - \Exp{f(x^K)}}{c\gamma K} + \br{\frac{1 - \gamma}{c\beta} + \frac{\gamma}{2c}} \sqn{h}_{\moL} + \frac{\gamma}{2c} \sigma^2.
\end{equation*}
where $\moL = \frac{1}{n} \sum_{i=1}^n \mL_i, h = \moL^{-1}\ob - \frac{1}{\sqrt{n}} \frac{1}{n} \sum_{i=1}^n \mD_i^{-\frac{1}{2}} \bb_i$ and $\ob = \frac{1}{n} \sum_{i=1}^n \bb_i$.
\end{theorem}

\begin{proof}
By using $\mL$-smoothness
\begin{eqnarray}
\Exp{f(x^{k+1}) \;|\; x^k}
& \overset{\eqref{eq:L-matrix-smooth}}{\leq} & 
f(x^k) - \gamma \left\langle \nabla f(x^k), \Exp{g^k} \right\rangle + \frac{\gamma^2}{2} \Exp{\sqn{g^k}_{\moL}} 
\notag \\& \overset{\eqref{eq:het_gen_estimator}, \eqref{eq:variance_decomposition}}{=} & 
f(x^k) - \gamma \left\langle \nabla f(x^k), \moL^{-1} \nabla f(x^k) + h \right\rangle
\notag \\&& \qquad + 
\frac{\gamma^2}{2} \br{\sqN{\Exp{g^k}}_{\moL} + \Exp{\sqN{g^k - \Exp{g^k}}_{\moL}}} \notag
\\& \overset{\eqref{eq:het_gen_estimator}}{=} &
f(x^k) - \gamma \br{\left\langle \nabla f(x^k), \moL^{-1} \nabla f(x^k) \right\rangle + \left\langle \nabla f(x^k), h \right\rangle} \notag
\\ && \qquad + 
\frac{\gamma^2}{2} \br{\sqN{\moL^{-1} \nabla f(x^k) + h}_{\moL} + \Exp{\sqN{g^k - \Exp{g^k}}_{\moL}}} \notag
\\ &\overset{\eqref{eq:dot_product}}{=}&
f(x^k) - \gamma \br{\sqN{\nabla f(x^k)}_{\moL^{-1}} + \left\langle \nabla f(x^k), h \right\rangle} + \frac{\gamma^2}{2} \Exp{\sqN{g^k - \Exp{g^k}}_{\moL}} \notag
\\ && \qquad + 
\frac{\gamma^2}{2} \br{\sqN{\nabla f(x^k)}_{\moL^{-1}} + 2\left\langle \nabla f(x^k), h \right\rangle + \sqn{h}_{\moL}} \notag
\\ &\leq&
f(x^k) - \gamma \br{1 - \nicefrac{\gamma}{2}} \sqN{\nabla f(x^k)}_{\moL^{-1}} + \frac{\gamma^2}{2} \sigma^2 \notag
\\&& \qquad - \gamma \br{1 - \gamma} \left\langle \nabla f(x^k), h \right\rangle + \frac{\gamma^2}{2} \sqn{h}_{\moL}, \notag 
\end{eqnarray}
where the last inequality follows from the grouping of similar terms and bounded heterogeneity
\begin{eqnarray} \label{eq:hetero_bound}
    \Exp{\sqN{g^k - \Exp{g^k}}_{\moL}} &=& \notag
    \Exp{\sqN{g^k - \br{\moL^{-1} \nabla f(x^k) + h}}_{\moL}} 
    \\&=&
    \Exp{\sqN{\moB^k x^k - \oCb - \br{x^k - \frac{1}{\sqrt{n}}\cDb}}_{\moL}} \leq \sigma^2.
\end{eqnarray}

Next, using a Fenchel-Young inequality \eqref{eq:fy-ineq_quad} for  $\left\langle \nabla f(x^k), -h \right\rangle$ and $1 - \gamma \geq 0$
\begin{eqnarray*}
\Exp{f(x^{k+1}) \;|\; x^k}
&\leq&
f(x^k) - \gamma \br{1 - \nicefrac{\gamma}{2}} \sqN{\nabla f(x^k)}_{\moL^{-1}} + \frac{\gamma^2}{2} \br{\sqn{h}_{\moL} +  \sigma^2} \notag 
\\&& \qquad + \gamma \br{1 - \gamma} \sbr{\beta \sqn{\nabla f(x^k)}_{\mcL^{-1}} + 0.25 \beta^{-1}\sqn{h}_{\mcL}} \notag
\\ &\leq&
f(x^k) - \gamma \br{1 - \nicefrac{\gamma}{2} - \beta \br{1 - \gamma}} \sqN{\nabla f(x^k)}_{\moL^{-1}} \notag
\\&& \qquad +
\gamma \cbr{\br{\beta^{-1} \br{1 - \gamma} + \frac{\gamma}{2}}\sqn{h}_{\moL} + \frac{\gamma}{2} \sigma^2},
\end{eqnarray*}
where in the last inequality we grouped similar terms and used the fact that $0.25 < 1$.

Now to guarantee that $1 - \nicefrac{\gamma}{2} - \beta(1 - \gamma) \geq c >0$, we choose the step size using
\begin{equation*}
    0 < \gamma \leq \gamma_{c, \beta} \eqdef \frac{1 - c - \beta}{\beta + 1/2},
\end{equation*}
where $\gamma_{c, \beta} > 0$ for $\beta + c < 1$. This means that $\beta$ can not arbitrarily grow to diminish $\beta^{-1}$. \\
Then, after standard manipulations and unrolling the recursion
\begin{equation*}
     \gamma c \sqN{\nabla f(x^k)}_{\moL^{-1}} \leq f(x^k) - \Exp{f(x^{k+1}) \;|\; x^k} + \gamma \br{\beta^{-1}\br{1 - \gamma} + \gamma/2} \sqn{h}_{\moL} + \frac{\gamma^2}{2} \sigma^2
\end{equation*}
we obtain
\begin{equation*}
    \frac{c}{K} \sum_{k=0}^{K-1} \Exp{\sqN{\nabla f(x^k)}_{\moL^{-1}}}
    \leq 
    \frac{f(x^0) - \Exp{f(x^K)}}{\gamma K} + \br{\beta^{-1}\br{1 - \gamma} + \gamma/2} \sqn{h}_{\moL} + \frac{\gamma}{2} \sigma^2.
\end{equation*}

\end{proof}

\subsubsection{Homogeneous case}
The main difference compared to the result in the previous subsection is that the gradient estimator expression \eqref{eq:hom_gen_estimator} holds deterministically (without expectation $\mathbb{E}$). That is why $g^k = \Exp{g^k}$ and heterogeneity term $\sigma^2$ equals to 0.

We provide the full statement and proof for the homogeneous result discussed in \ref{sec:gen_convergence}.
\begin{theorem} \label{thm:gen_gradient}
Consider the method \eqref{eq:SGD_generic} with estimator \eqref{eq:hom_gen_estimator} for a homogeneous quadratic problem \eqref{eq:het_gen_quad_problem} with positive-definite matrix $\mL_i \equiv \mL \succ 0$. Then if exists $\mD^{-\frac{1}{2}}$ for $\mD \eqdef \Diag(\mL)$, scaled permutation sketch $\mC'_i = \sqrt{n} e_{\pi_i} e_{\pi_i}^\top$ is used and the step size is chosen as
\begin{equation*}
    0 < \gamma \leq \gamma_{c, \beta} \eqdef \frac{1 - c - \beta}{\beta + 1/2},
\end{equation*}
where $\gamma_{c, \beta} > 0$ for $\beta + c < 1$. Then the iterates satisfy
\begin{equation} \label{eq:gen_grad_res}
    \frac{1}{K} \sum_{k=0}^{K-1} \Exp{\sqN{\nabla f(x^k)}_{\mcL^{-1}}}
    \leq 
    \frac{f(x^0) - \Exp{f(x^K)}}{c \gamma K} + \br{\frac{1 - \gamma}{c\beta} + \frac{\gamma}{2c}} \sqn{h}_{\mcL},
\end{equation}
where $\mcL = \mD^{-\frac{1}{2}} \mL \mD^{-\frac{1}{2}}, h = \mcL^{-1}\cb - \frac{1}{\sqrt{n}}\cb$ and $\cb = \mD^{-\frac{1}{2}} \bb$.
\end{theorem}

\begin{proof}
By using $\mL$-smoothness
\begin{eqnarray*}
\Exp{f(x^k - \gamma g^k) \;|\; x^k} 
    & \overset{\eqref{eq:L-matrix-smooth}}{\leq} &
    f(x^k) - \left\langle\nabla f(x^k), \gamma \Exp{g^k} \right\rangle + \frac{\gamma^2}{2} \Exp{\sqN{g^k}_{\mcL}} \\
& \leq & f(x^k) - \gamma \left\langle \nabla f(x^k), \mcL^{-1} \nabla f(x^k) + h \right\rangle + \frac{\gamma^2}{2} \sqN{\mcL^{-1} \nabla f(x^k) + h}_{\mcL}
 \\ &\overset{\eqref{eq:dot_product}}{=}&
 f(x^k) - \gamma \br{\left\langle \nabla f(x^k), \mcL^{-1} \nabla f(x^k) \right\rangle + \left\langle \nabla f(x^k), h \right\rangle} \notag
 \\ && \qquad + 
 \frac{\gamma^2}{2} \br{\sqN{\nabla f(x^k)}_{\mcL^{-1}} + 2\left\langle \nabla f(x^k), h \right\rangle + \sqn{h}_{\mcL}} \notag
 \\ &=&
 f(x^k) - \gamma \br{1 - \nicefrac{\gamma}{2}} \sqN{\nabla f(x^k)}_{\mcL^{-1}}
+ \frac{\gamma^2}{2} \sqn{h}_{\mcL} - \gamma \br{1 - \gamma} \left\langle \nabla f(x^k), h \right\rangle \notag
\end{eqnarray*}
Next by using a Fenchel-Young inequality \eqref{eq:fy-ineq_quad} for  $\left\langle \nabla f(x^k), -h \right\rangle$ and $1 - \gamma \geq 0$

\begin{eqnarray*}
\Exp{f(x^{k+1}) \;|\; x^k}
& \leq & 
f(x^k) - \gamma \br{1 - \nicefrac{\gamma}{2}} \sqN{\nabla f(x^k)}_{\mcL^{-1}}  + \frac{\gamma^2}{2} \sqn{h}_{\mcL}
\\ && \qquad + \gamma \br{1 - \gamma} \sbr{\beta \sqn{\nabla f(x^k)}_{\mcL^{-1}} + 0.25\beta^{-1}\sqn{h}_{\mcL}}
\\ & = & 
f(x^k) - \gamma \br{1 - \nicefrac{\gamma}{2} - \beta(1 - \gamma)} \sqN{\nabla f(x^k)}_{\mcL^{-1}}
\\ && \qquad + \gamma \br{\beta^{-1}\br{1 - \gamma} + \gamma/2} \sqn{h}_{\mcL}.
\end{eqnarray*}
Now to guarantee that $1 - \nicefrac{\gamma}{2} - \beta(1 - \gamma) \geq c > 0$ we choose the step size as
\begin{equation*}
    0 < \gamma \leq \gamma_{c, \beta} \eqdef \frac{1 - c - \beta}{\beta + 1/2},
\end{equation*}
where $\gamma_{c, \beta} \geq 0$ for $\beta + c < 1$.
\\
Then after standard manipulations and unrolling the recursion
\begin{equation} \label{eq:descent_recursion}
     \gamma c \sqN{\nabla f(x^k)}_{\mcL^{-1}} \leq f(x^k) - \Exp{f(x^{k+1}) \;|\; x^k} + \gamma \br{\beta^{-1}\br{1 - \gamma} + \gamma/2} \sqn{h}_{\mcL}
\end{equation}
we obtain the formulated result
\begin{equation*}
    \frac{c}{K} \sum_{k=0}^{K-1} \Exp{\sqN{\nabla f(x^k)}_{\mcL^{-1}}}
    \leq 
    \frac{f(x^0) - \Exp{f(x^K)}}{\gamma K} + \br{\beta^{-1}\br{1 - \gamma} + \gamma/2} \sqn{h}_{\mcL}.
\end{equation*}
\end{proof}

\begin{remark}
\textbf{1)} The first term in the convergence upper bound \eqref{eq:gen_grad_res} is minimized by maximizing product $c\cdot\gamma$, which motivates to choose $c>0$ and $\gamma \leq 1$ as large as possible. Although due to the constraint on the step size (and $\beta > 0$)
\begin{equation*}
    0 < \gamma \leq \gamma_{c, \beta} \eqdef \frac{1 - c - \beta}{\beta + 1/2},
\end{equation*}
constant $c \in (0, 1)$. So, by maximizing $c$ the value $\gamma_{c, \beta}$ becomes smaller, thus there is a trade-off.

\textbf{2)} The second term or the neighborhood size (multiplier in front of $\sqn{h}_{\mcL}$)
\begin{equation*}
    \Psi(\beta, \gamma) \eqdef \frac{\beta^{-1}\br{1 - \gamma} + \gamma/2}{c} = \frac{\beta^{-1}\br{1 - \gamma} + \gamma/2}{1 - \gamma/2 - \beta (1 - \gamma)}
\end{equation*}
can be numerically minimized (e.g.~by using WolframAlpha) with constraints $\gamma\in(0, 1]$ and $\beta > 0$. The solution of such optimization problem is $\gamma^\star \approx 1$ and $\beta^\star \approx \xi \in \{3.992, 2.606, 2.613\}$. In fact, $\Psi(\beta^\star, \gamma^\star) \approx 0.5$.
\end{remark}

\paragraph{Functional gap convergence.}
Note that for the quadratic optimization problem \eqref{eq:het_gen_quad_problem}
\begin{equation} \label{eq:grad_func}
    \sqN{\nabla f(x^k)}_{\mcL^{-1}} = \left\langle \mcL x^k - \cb, \mcL^{-1} \br{\mcL x^k - \cb} \right\rangle = 2 \br{f(x^k) - f(x^\star)}.
\end{equation}
Then by rearranging and subtracting $f^\star \eqdef f(x^\star)$ from both sides of inequality \eqref{eq:descent_recursion} we obtain
\begin{eqnarray*}
\Exp{f(x^{k+1}) \;|\; x^k} - f^\star 
&\leq& 
f(x^k) - f^\star - \gamma c \sqN{\nabla f(x^k)}_{\mcL^{-1}} + \gamma \br{\beta^{-1}\br{1 - \gamma} + \gamma/2} \sqn{h}_{\mcL}
\\ &\overset{\eqref{eq:grad_func}}{=}&
\br{f(x^k) - f^\star} - \gamma c \cdot 2\br{f(x^k) - f^\star} + \gamma \br{\beta^{-1}\br{1 - \gamma} + \gamma/2} \sqn{h}_{\mcL}
\\\ &=&
\br{1 - 2 \gamma c} \br{f(x^k) - f^\star} + \gamma \br{\beta^{-1}\br{1 - \gamma} + \gamma/2} \sqn{h}_{\mcL}.
\end{eqnarray*}
After unrolling the recursion 
\begin{eqnarray*}
\Exp{f(x^{k+1}) \;|\; x^k} - f^\star 
&\leq& 
\br{1 - 2 \gamma c}^k \br{f(x^0) - f^\star} + \gamma \br{\beta^{-1}\br{1 - \gamma} + \gamma/2} \sqn{h}_{\mcL} \sum_{i=0}^k \br{1 - 2 \gamma c}^i
\\\ &\leq&
\br{1 - 2 \gamma c}^k \br{f(x^0) - f^\star} + \frac{1}{2c} \br{\beta^{-1}\br{1 - \gamma} + \gamma/2} \sqn{h}_{\mcL}.
\end{eqnarray*}

This result is formalized in the following Theorem.
\begin{theorem} \label{thm:gen_function}
Consider the method \eqref{eq:SGD_generic} with estimator \eqref{eq:hom_gen_estimator} for a homogeneous quadratic problem \eqref{eq:het_gen_quad_problem} with positive-definite matrix $\mL_i \equiv \mL \succ 0$. Then if exists $\mD^{-\frac{1}{2}}$ for $\mD \eqdef \Diag(\mL)$, scaled permutation sketch $\mC'_i = \sqrt{n} e_{\pi_i} e_{\pi_i}^\top$ is used and the step size is chosen as
\begin{equation*}
    0 < \gamma \leq \gamma_{c, \beta} \eqdef \frac{1 - c - \beta}{\beta + 1/2},
\end{equation*}
where $\gamma_{c, \beta} >0$ for $\beta + c < 1$. Then the iterates satisfy
\begin{equation} \label{eq:gen_func_res}
    \Exp{f(x^k)} - f^\star \leq \br{1 - 2 \gamma c}^k \br{f(x^0) - f^\star} + \frac{1}{2c} \br{\beta^{-1}\br{1 - \gamma} + \gamma/2} \sqn{h}_{\mcL},
\end{equation}
where $h = \mcL^{-1}\cb - \frac{1}{\sqrt{n}}\cb$ and $\mcL = \mD^{-\frac{1}{2}} \mL \mD^{-\frac{1}{2}}, \cb = \mD^{-\frac{1}{2}} \bb$.
\end{theorem}

This result shows that for a proper choice of the step size $\gamma = 1$ and constant $c=1/2$, the functional gap can converge in basically one iteration to the neighborhood of size
\begin{equation*}
    \sqn{h}_{\mcL} = \left\langle \mcL \br{\mcL^{-1}\cb - \frac{1}{\sqrt{n}}\cb}, \mcL^{-1} \cb - \frac{1}{\sqrt{n}}\cb \right\rangle,
\end{equation*}
which equals zero if $\mcL^{-1} \cb = \frac{1}{\sqrt{n}}\cb$. This condition is the same as the condition we obtained at the end of Subsection \ref{sec:gen_convergence} with asymptotic analysis of the iterates in the homogeneous case.

\paragraph{Discussion of the trace.}
Consider a positive-definite $\mL \succ 0$ such that $\exists \mD^{-\frac{1}{2}}$. Thus $\mcL = \mD^{-\frac{1}{2}} \mL \mD^{-\frac{1}{2}}$ has only ones on the diagonal and $\tr(\mcL) = n$. Then
\begin{equation*}
    n \cdot \tr(\mcL^{-1}) = \tr(\mcL) \tr(\mcL^{-1}) = \br{\lambda_1 + \dots + \lambda_n} \br{\frac{1}{\lambda_1} + \dots + \frac{1}{\lambda_n}} \geq n^2,
\end{equation*}
where the last inequality is due to the relation between harmonic and arithmetic means. Therefore $\tr(\mcL^{-1}) = \lambda_1^{-1} + \dots + \lambda_n^{-1} \geq n$ and sum of $\mcL^{-1}$ eigenvalues has to be greater than $n$.

\subsection{Generalization to $n \neq d$ case.}
Our results can be generalized in a similar way as in \cite{szlendak2022permutation}.

\textbf{1)} $d = qn$, for integer $q \geq 1$. Let $\pi = \br{\pi_1, \dots, \pi_d}$ be a random permutation of $\{1, \dots,d\}$. Then for each $i \in \{1, \dots, n\}$ define
\begin{equation*}
    \mC_i' := \sqrt{n} \cdot \sum_{j=q(i-1)+1}^{q i} e_{\pi_j} e_{\pi_j}^{\top}.
\end{equation*}
Matrix $\Exp{\moB^k}$ for the homogeneous preconditioned case can be computed as follows:
\begin{eqnarray*}
    \Exp{\moB^k} &=& 
    \Exp{\frac{1}{n} \sum_{i=1}^n \mC_i' \mcL \mC_i'} 
    \\&=& 
    \frac{1}{n} \sum_{i=1}^n\Exp{\sum_{j=q(i-1)+1}^{q i} n e_{\pi_j} e_{\pi_j}^{\top} \mcL e_{\pi_j} e_{\pi_j}^{\top}} 
    \\&=&
    \sum_{i=1}^n \sum_{j=q(i-1)+1}^{q i} \Exp{e_{\pi_j} e_{\pi_j}^{\top} \mcL e_{\pi_j} e_{\pi_j}^{\top}} 
    \\&=&
    \sum_{i=1}^n \sum_{j=q(i-1)+1}^{q i} \frac{1}{d} \sum_{l=1}^d e_{l} e_{l}^{\top} \mcL e_{l} e_{l}^{\top} 
    \\&=&
    \sum_{i=1}^n \sum_{j=q(i-1)+1}^{q i} \frac{1}{d} \Diag(\mcL)
    \\&=&
    n \frac{q}{d} \Diag(\mcL)
    \\&=&
    \Diag(\mcL) \\&=& \mI.
\end{eqnarray*}
As for the linear term  
\begin{eqnarray*}
\Exp{\mC'\bb} &=& 
\Exp{\frac{1}{n} \sum_{i=1}^n \mC'_i \cb} = 
\frac{1}{n} \sum_{i=1}^n \Exp{\sum_{j=q(i-1)+1}^{q i} \sqrt{n} e_{\pi_j} e_{\pi_j}^{\top} \cb} 
\\&=& 
\frac{1}{\sqrt{n}} \sum_{i=1}^n \sum_{j=q(i-1)+1}^{q i} \frac{1}{d} \mI \cb = \frac{\sqrt{n}q}{d} \mI \cb = \frac{1}{\sqrt{n}} \cb.
\end{eqnarray*}

\textbf{2)} $n = qd$, for integer $q \geq 1$. Define the multiset $S \eqdef \{1, \ldots, 1, 2, \ldots, 2, \ldots, d, \ldots, d\}$, where each number occurs precisely $q$ times. Let $\pi = \br{\pi_1, \ldots, \pi_n}$ be a random permutation of $S$. Then for each $i \in \{1, \ldots, n\}$ define
\begin{equation*}
    \mC_i' := \sqrt{d} \cdot e_{\pi_i} e_{\pi_i}^{\top}.
\end{equation*}

\begin{eqnarray*}
    \Exp{\moB^k} &=& 
    \Exp{\frac{1}{n} \sum_{i=1}^n \mC_i' \mcL \mC_i'} 
    = 
    \frac{1}{n} \sum_{i=1}^n \Exp{d e_{\pi_i} e_{\pi_i}^{\top} \mcL e_{\pi_i} e_{\pi_i}^{\top}} 
    \\&=& 
    \frac{1}{n} \sum_{i=1}^n \frac{1}{d} \sum_{j=1}^d d e_{j} e_{j}^{\top} \mcL e_{j} e_{j}^{\top}
    =
    \frac{1}{n} \sum_{i=1}^n \Diag(\mcL)
    = 
    \mI.
\end{eqnarray*}
The linear term  
\begin{eqnarray*}
    \Exp{\mC'\bb}=
    \Exp{\frac{1}{n} \sum_{i=1}^n \mC'_i \cb} = 
    \frac{1}{n} \sum_{i=1}^n \Exp{\sqrt{d} e_{\pi_i} e_{\pi_i}^{\top} \cb} 
    =
    \frac{\sqrt{d}}{n} \sum_{i=1}^n \frac{1}{d} \mI \cb = \frac{1}{\sqrt{d}} \cb.
\end{eqnarray*}

To sum up both cases, in a homogeneous preconditioned setting 
$\Exp{\moB^k} = \mI$ and 
$$\Exp{\mC'\bb} = \Exp{\frac{1}{n} \sum_{i=1}^n \mC'_i \bb} = \cb/\sqrt{\min(n, d)}.$$

Similar modifications and calculations can be performed for heterogeneous scenarios.
The case when $n$ does not divide $d$ and vice versa is generalized using constructions from \cite{szlendak2022permutation}.

\section{Generalization Beyond Quadratics} \label{sec:generalization}

In this section our analysis is revisited for the more general case of smooth (non-convex) functions and class of compressors. While this result may not be as nuanced as for the quadratic model, it still does not require restrictive assumptions on gradient estimator, unlike prior works.

\subsection{Preliminary Facts}
Our convergence analysis relies on the following Lemma.

\begin{lemma}[Descent Lemma \cite{li2021page}] \label{lem:descent_lemma_appendix}
Suppose that function $f$ is $L$-smooth \eqref{eq:L-matrix-smooth} and let $x^{k+1} \eqdef x^k - \gamma g^k$. Then for any $g^k \in \mathbb{R}^{d}$ and $\gamma>0$, we have
\begin{equation*} \label{eq:descent_lemma}
    f(x^{k+1}) \leq 
    f(x^k) - \frac{\gamma}{2} \sqN{\nabla f(x^k)}
    - \br{\frac{1}{2\gamma} - \frac{L}{2}} \sqn{x^{k+1} - x^k}
    + \frac{\gamma}{2} \sqN{g^k - \nabla f(x^k)},
\end{equation*}
\end{lemma}

If a continuously differentiable function $f$ is $L$-smooth \eqref{eq:L-matrix-smooth}, then for any $x, y \in \mathbb{R}^{d}$, it is satisfied
\begin{equation} \label{eq:smoothness}
    \norm{\nabla f(x) - \nabla f(y)} \leq L \norm{x - y}.
\end{equation}

\begin{lemma}[\citet{khaled2022better}]
    Let $f$ be $L$-smooth \eqref{eq:L-matrix-smooth} and $f^{\inf}$-lower bounded. Then for any $x\in \bR^d$ we have:
    \begin{equation} \label{eq:grad_f_lemma}
        \sqN{\nabla f(x)} \leq 2 L (f(x) - f^{\inf})
    \end{equation}
\end{lemma}

We also use introduced by \citet{szlendak2022permutation} notion of $AB$ inequality for a collection of (correlated) compressors $\cC_i$.

\begin{definition}[$AB$ inequality]  \label{def:AB_inequality}
A collection of random unbiased operators $\cC_1, \dots, \cC_n : \bR^d \to \bR^d$  satisfies $AB$ inequality if there exist constants $A, B \geq 0$ such that
\begin{equation} \label{eq:AB_inequality} 
\begin{aligned}
    \E\sqN{ \frac{1}{n}\sum_{i=1}^n \cC_i(x_i) - \frac{1}{n}\sum_{i=1}^n x_i}  &\leq A \frac{1}{n}\sum_{i=1}^n \norm{x_i}^2 - B \sqN{\frac{1}{n}\sum_{i=1}^n x_i}
\end{aligned}
\end{equation}
for all $x_1, \dots, x_n \in \bR^d$. For brevity: $\{\cC_i\}_{i=1}^n \in \mathbb{C}(A,B)$.
\end{definition}
This definition is handy for analyzing algorithms like IST as it allows to generalize analysis for the case when compressors $\cC_1, \dots, \cC_n$ are dependent, which happens when the parameters of the model are randomly decomposed in a non-overlapping fashion. Moreover, $AB$ inequality is satisfied for unbiased compressors \eqref{def:unbiased_compressor}: $\cC_i \in \bU(\omega_i)$ for all $i$.

Denote $w_i^k \eqdef \cC_i^k (x^k)$.
If compressors $\cC_i^k$ satisfy the \emph{perfect reconstruction} property $\frac{1}{n} \sum_{i=1}^n \cC^k_i(x) = x$, which makes sense for the IST formulation, then method \eqref{eq:IST} can be reformulated in the following way
\begin{eqnarray} \label{eq:IST_general}
    x^{k+1} &=& 
    \frac{1}{n} \sum_{i=1}^n \br{w_i^k - \gamma \cQ_i^k \br{\nabla f_i(w_i^k)}} \notag
    \\ &=&
    \frac{1}{n} \sum_{i=1}^n \cC_i^k (x^k) - \gamma \frac{1}{n} \sum_{i=1}^n \cQ_i^k \br{\nabla f_i(w_i^k))} \notag
    \\ &=&
    x^k - \gamma \underbrace{\frac{1}{n} \sum_{i=1}^n \cQ_i^k \br{\nabla f_i(\cC_i^k (x^k))}}_{g^k}.
\end{eqnarray}
This reformulation makes the algorithm amenable to analysis using Lemma \ref{lem:descent_lemma_appendix}.

\subsection{Convergence Analysis}

\begin{theorem}
Let $f$ and $f_i$ be $L$ and $L_i$-smooth \eqref{eq:L-matrix-smooth} respectively. Collection of unbiased compressors $\cC_i^k, \cQ_i^k$ satisfy $AB$-inequality \eqref{eq:AB_inequality} and property $\frac{1}{n} \sum_{i=1}^n \cC^k_i(x) = x$ for every $k$. Then for step size chosen as 
$$\gamma \leq \min \left\{\frac{1}{L}, \frac{\sqrt{1 + \nicefrac{2}{(A L_{\max} K)}} - 1}{2}\right\}$$
the iterations of Algorithm~\eqref{eq:IST_general} for any $K \geq 1$ satisfy 
\begin{equation} \label{eq:general_res}
    \min_{0\leq k \leq K-1} \sqN{\nabla f(x^k)} 
    \leq 
    \frac{6 \br{f(x^0) - \finf}}{\gamma K} + 2\br{1+(1+\beta^{-1})A} \max_k \sqn{x^k} \overline{L^2_\omega}
    + 4 A L_{\max} (1+\beta) \Delta,
\end{equation}
where $\beta \leq \gamma, \overline{L^2_\omega} = \frac{1}{n} \sum_{i=1}^n L_i^2 \omega_i, L_{\max} = \max_i L_i$, and $\Delta = \frac{1}{n} \sum_{i=1}^n \br{\finf - f_i^{\inf}}$.
\end{theorem}
Note that minimum over squared gradient norm \(\sqN{\nabla f(x^k)}\) and maximum over model weights \(\sqn{x^k}\) in \eqref{eq:general_res} can be replaced with with weighted sums due to proof step \eqref{eq:weighted_sum}.
Let us contrast the obtained convergence result with what we have for the quadratic case in \eqref{eq:het_gen_grad_res}. The first term is basically the same and decreases as $1/K$ for constant step size $\gamma$. The other terms are different as they involve the norm of the iterates $\|x^k\|^2$ throughout training. However, they also can not be decreased by diminishing the step size, which is similar to the quadratic case. Thus, we showed that for the more general case of smooth losses, IST can converge to the irreducible neighborhood of the stationary point. The obtained conclusions highlight the generalizability of our insights beyond the quadratic model.

Moreover, there is an additional part representing the heterogeneity of the distributed problem proportional to $\Delta$. In contrast to the term involving $\|x^k\|^2$ it can be eliminated in the homogeneous case when functions have a shared \say{minimizer} $f_i^{\inf} = \finf$, which may hold in the overparametrized regime.

\textbf{Contrasting to prior works.} \citet{khaled2019gradient} analyzed a similar method with unbiased compression in the single node and strongly convex setting. \citet{yuan2022distributed} extended their results to a finite-sum case with random sampling of one client. However, prior works suffer from a very strong condition on the sparsification variance $\omega \lesssim \nicefrac{\mu}{L}$, which is not the case for our analysis. Another important difference is that convergence bounds in \cite{khaled2019gradient} and \cite{yuan2022distributed} depend on the norm of the optimal solution $x^\star = \argmin f(x)$, which may not exist in the non-convex setting. 

\begin{proof}
We start with the Descent Lemma \eqref{eq:descent_lemma}

\begin{eqnarray} \label{eq:descent_lemma_}
    f(x^{k+1}) &\leq&
    f(x^k) - \frac{\gamma}{2} \sqN{\nabla f(x^k)}
    - \br{\frac{1}{2\gamma} - \frac{L}{2}} \sqn{x^{k+1} - x^k}
    + \frac{\gamma}{2} \sqN{g^k - \nabla f(x^k)} \notag
    \\ &\leq&
    f(x^k) - \frac{\gamma}{2} \sqN{\nabla f(x^k)}
    + \frac{\gamma}{2} \sqN{g^k - \nabla f(x^k)},
\end{eqnarray}
where the step size is chosen as $\gamma \leq 1/L$.

Next prove an auxiliary result needed to work with the last term of \eqref{eq:descent_lemma_}
\begin{eqnarray} \label{eq:ES_distributed}
    \E \sqN{\frac{1}{n} \sum_{i=1}^n \nabla f_i(\cC_i^k (x^k)) - \nabla f_i(x^k)}
    &\overset{\eqref{eq:average_inequality}}{\leq}&
    \frac{1}{n} \sum_{i=1}^n \E \sqN{\nabla f_i(\cC_i^k (x^k)) - \nabla f_i(x^k)} \notag
    \\ &\overset{\eqref{eq:smoothness}}{\leq}&
    \frac{1}{n} \sum_{i=1}^n L_i^2 \E \sqN{\cC_i^k(x^k) - x^k} \notag
    \\&\overset{\eqref{eq:unbiased_compressor}}{\leq}&
    \frac{1}{n} \sum_{i=1}^n L_i^2 \omega_i \sqn{x^k} \notag
    \\ &=&
    \overline{L^2_\omega} \sqn{x^k},
\end{eqnarray}
for $\overline{L^2_\omega} \eqdef \frac{1}{n} \sum_{i=1}^n L_i^2 \omega_i$.

Recall the expression for $g^k = \frac{1}{n} \sum_{i=1}^n \cQ_i^k \br{\nabla f_i(w_i^k)}$ and $w_i^k \eqdef \cC_i^k (x^k)$. Now we can upper bound the last term of \eqref{eq:descent_lemma_}
\begin{eqnarray*}
    \E \sqN{g^k - \nabla f(x^k)} &=&
    \E \sqN{\frac{1}{n} \sum_{i=1}^n \cQ_i^k\br{\nabla f_i(w_i^k)} - \nabla f_i(x^k)} \\
    &\overset{\eqref{eq:triangle}}{\leq}&
    2 \E \sqN{\frac{1}{n} \sum_{i=1}^n \cQ_i^k\br{\nabla f_i(w_i^k)} - \nabla f_i(w_i^k)} +
    2 \E \sqN{\frac{1}{n} \sum_{i=1}^n \nabla f_i(w_i^k) - \nabla f_i(x^k)} \\
    &\overset{\eqref{eq:AB_inequality}}{\leq}&
    2 \sbr{A \frac{1}{n} \sum_{i=1}^n \E \sqN{ \nabla f_i(w_i^k)} - B \E \sqN{\frac{1}{n} \sum_{i=1}^n \nabla f_i(w_i^k)}} 
    \\ && \qquad +
    2 \E \sqN{\frac{1}{n} \sum_{i=1}^n \nabla f_i(\cC_i^k (x^k)) - \nabla f_i(x^k)} \\
    &\overset{\eqref{eq:ES_distributed}}{\leq}&
    2 A \frac{1}{n} \sum_{i=1}^n \E \sqN{\nabla f_i(\cC_i^k (x^k)) \pm \nabla f_i(x^k)} +
    2 \frac{1}{n} \sum_{i=1}^n L_i^2 \omega_i \sqn{x^k} 
    \\&\overset{\eqref{eq:triangle}}{\leq}&
    2 A \frac{1}{n} \sum_{i=1}^n\sbr{\br{1 + \beta^{-1}} \E \sqN{\nabla f_i(\cC_i^k (x^k)) - \nabla f_i(x^k)} +  \br{1 + \beta} \sqN{\nabla f_i(x^k)}} 
     +
    2 \overline{L^2_\omega} \sqn{x^k} 
    \\&\overset{\eqref{eq:ES_distributed}}{\leq}&
    2 A \frac{1}{n} \sum_{i=1}^n \sbr{\br{1 + \beta^{-1}} L_i^2 \omega_i \sqn{x^k} + \br{1 + \beta} \sqN{\nabla f_i(x^k)}} +
    2 \overline{L^2_\omega} \sqn{x^k} \\
    &\overset{\eqref{eq:grad_f_lemma}}{\leq}&
    2 \br{1 + A \br{1 + \beta^{-1}}} \overline{L^2_\omega} \sqn{x^k} + 2 A \br{1 + \beta} \frac{1}{n} \sum_{i=1}^n 2 L_i \sbr{f_i(x^k) - f_i^{\inf}} \\
    &=&
    2 \br{1 + A \br{1 + \beta^{-1}}} \overline{L^2_\omega} \sqn{x^k} + 4 A \br{1 + \beta} \frac{1}{n} \sum_{i=1}^n L_i \sbr{f_i(x^k) - \finf - f_i^{\inf} + \finf} \\
    &\leq& 
    2 \br{1 + A \br{1 + \beta^{-1}}} \overline{L^2_\omega} \sqn{x^k} + 4 A \br{1 + \beta} L_{\max} \frac{1}{n} \sum_{i=1}^n \sbr{f_i(x^k) - \finf + \finf - f_i^{\inf}} \\
    &=&
    2 \br{1 + A \br{1 + \beta^{-1}}} \overline{L^2_\omega} \sqn{x^k} + 4 A \br{1 + \beta} L_m \Big[f(x^k) - \finf + \underbrace{\frac{1}{n} \sum_{i=1}^n \br{\finf - f_i^{\inf}}}_{\Delta}\Big],
\end{eqnarray*}
where $L_{m} = L_{\max} = \max L_i$.

Combined with \eqref{eq:descent_lemma_}
and denoting $\delta^k \eqdef f(x^k) - \finf$ it leads to

\begin{equation*}
\begin{aligned}
    \delta^{k+1} 
    &\leq 
    \delta^k - \frac{\gamma}{2} \sqN{\nabla f(x^k)} + 2 \gamma A \br{1 + \beta} L_m \br{f(x^k) - \finf} \\
    &\qquad + 
    \frac{\gamma}{2} \sbr{2 \br{1 + A \br{1 + \beta^{-1}}} \overline{L^2_\omega} \sqn{x^k} + 4 A \br{1 + \beta} L_m \Delta} \\
    &=
    \br{1 + 2 \gamma A L_m \br{1 + \beta}} \delta^k - \frac{\gamma}{2} \sqN{\nabla f(x^k)} 
    \\
    &\qquad + 
    \frac{\gamma}{2} \underbrace{\sbr{2 \br{1 + A \br{1 + \beta^{-1}}} \overline{L^2_\omega} \sqn{x^k} + 4 A \br{1 + \beta} L_m \Delta}}_{C^k}.
\end{aligned}
\end{equation*}

After rearranging we obtain
\begin{equation} \label{eq:init_recursion}
\begin{aligned}
    \sqN{\nabla f(x^k)} \leq 
    \frac{2}{\gamma} \big(1 + \underbrace{2 \gamma A \br{1 + \beta} L_m}_{D} \big) \delta^k - \frac{2}{\gamma} \delta^{k+1} + C^k.
\end{aligned}
\end{equation}

Next by following technique by \citet{stich2019unified} we introduce an exponentially decaying weighting sequence 
\begin{equation*}
    w^k = \frac{w^{k-1}}{1 + D} \leq \dots \leq w^{-1}.   
\end{equation*}

Multiplying recursion \eqref{eq:init_recursion} by $w^k$, we get
\begin{equation*}
\begin{aligned} 
    w^k  \sqN{\nabla f(x^k)} & \leq 
    \frac{2w^k\left(1 + D\right)}{\gamma} \delta^{k} - \frac{2w^k}{\gamma} \delta^{k+1} + w^k C^k \\ & =
    \frac{2 w^{k-1}}{\gamma} \delta^{k} - \frac{2 w^k}{\gamma} \delta^{k+1} + w^k C^k
\end{aligned}
\end{equation*}
After summing up both sides for $k$ from 0 to $K-1$, and telescoping terms
\begin{equation} \label{eq:sum_recursion}
    \sum_{k=0}^{K-1} w^k \sqN{\nabla f(x^k)} \leq \frac{2 w^{-1}}{\gamma} \delta^{0} - \frac{2 w^{K-1}}{\gamma} \delta^{K} + \sum_{k=0}^{K-1} w^k C^k.
\end{equation}
Now define $W^K \eqdef \sum_{k=0}^{K-1} w^k$ and divide both sides of \eqref{eq:sum_recursion} by $W^K$
\begin{equation} \label{eq:weighted_sum}
    \min_{0\leq k \leq K-1} \sqN{\nabla f(x^k)} \leq
    \frac{1}{W^k} \sum_{k=0}^{K-1} w^k \sqN{\nabla f(x^k)} \leq \frac{2 w^{-1}}{W^k\gamma} \delta^{0} - \frac{2 w^{K-1}}{W^k\gamma} \delta^{K} + \frac{1}{W^k} \sum_{k=0}^{K-1} w^k C^k.
\end{equation}
By using the fact that 
\begin{equation*}
    W^k = \sum_{k=0}^{K-1} w^k \geq \sum_{k=0}^{K-1} \min_{0 \leq i \leq K-1} w^i = K w^{K-1} = \frac{K w^{-1}}{(1 + D)^K},
\end{equation*}
we obtain 
\begin{equation} \label{eq:exponential_bound}
\begin{aligned}
    \min_k \sqN{\nabla f(x^k)} 
    &\leq 
    2\frac{\br{1 + 2 \gamma A \br{1 + \beta} L_m}^K}{\gamma K} \delta^0 + \max_k C^k.
\end{aligned}
\end{equation}
Next to simplify the obtained upper bound we use the fact that $1 + x \leq \exp (x)$
\begin{equation*}
\left(1 + 2 \gamma A \br{1 + \beta} L_m\right)^{K} \leq\left(\exp \left(2 \gamma A \br{1 + \beta} L_m\right)\right)^{K} = \exp \left(2 \gamma A \br{1 + \beta} L_m K\right) \leq \exp (1) \leq 3,
\end{equation*}
where the second inequality for $\beta \leq \gamma$ holds if 
$$\underbrace{2 A L_m K}_{A'_K} \gamma \br{1 + \gamma} \leq 1.$$
Then condition $A'_K \gamma + A'_K \gamma^2 -1 \leq 0$ holds for
\begin{equation*}
    \gamma \leq \frac{\sqrt{1 + 4/A'_K} - 1}{2} = 
    \frac{\sqrt{1 + \nicefrac{2}{(A L_m K)}} - 1}{2}.
\end{equation*}

As a result \eqref{eq:exponential_bound} leads to
\begin{equation*}
\begin{aligned}
    \min_k \sqN{\nabla f(x^k)} 
    \leq 
    \frac{6 \br{f(x^0) - \finf}}{\gamma K} + 
    2 \br{1 + A \br{1+\beta^{-1}}} \overline{L^2_\omega} \max_k \sqn{x^k} + 4 A \br{1 + \beta} L_m \Delta,
\end{aligned}
\end{equation*}
where $\Delta = \frac{1}{n} \sum_{i=1}^n \br{\finf - f_i^{\inf}}$ and $\overline{L^2_\omega} = \frac{1}{n} \sum_{i=1}^n L_i^2 \omega_i$.

\end{proof}

\section{Comparison to related works} \label{sec:app_comparison}

\paragraph{Overview of theory provided in the original IST work \cite{yuan2022distributed}.}

The authors consider the following method
\begin{equation} \label{eq:orig_IST_method}
    x^{k+1} = \cC (x^k) - \gamma \nabla f_{i_k} (\cC(x^k)),
\end{equation}

where $[\cC(x)]_i = x_i \cdot \mathcal{B}e (p)$\footnote{$\mathcal{B}_p (x) \eqdef \left\{\begin{array}{ll} x/p & \text { with probability } p \\ 0 & \text { with probability } 1-p\end{array}\right.$} is a Bernoulli sparsifier and $i_k$ is sampled uniformly at random from $[n]$.

The analysis in \cite{yuan2022distributed} relies on the assumptions

\begin{enumerate}
    \item $L_i$-smoothness of individual losses $f_i$;
    \item $Q$-Lipschitz continuity of $f$: $|f(x) - f(y)| \leq Q \norm{x - y}$;
    \item Error bound (or PŁ-condition): $\norm{\nabla f(x)} \geq \mu \norm{x^\star - x}$, where $x^\star$ is the global optimum;
    \item Stochastic gradient variance: $\Exp{\sqN{\nabla f_{i_k} (x)}} \leq M + M_f \sqN{\nabla f(x)}$;
    \item $\Exp{\nabla f_{i_k}(\cC(x^k)) \,|\, x^k} = \nabla f (x^k) + \eps, \quad \norm{\eps} \leq B$.
\end{enumerate}

Convergence result \cite[Theorem 1]{yuan2022distributed} for step size $\gamma=\nicefrac{1}{(2 L_{\max})}$:
\begin{equation*}
    \min_{k \in \{1,\dots,K\}} \Exp{\sqN{\nabla f (x^k)}} \leq \frac{f(x^0) - f(x^\star)}{\alpha (K+1)} + \frac{1}{\alpha} \cdot \br{\frac{B Q}{2 L_{\max}} + \frac{5 L_{\max} \omega}{2} \sqn{x^\star} + \frac{M}{4 L_{\max}}},
\end{equation*}
where $\alpha \eqdef \frac{1}{2 L_{\max}} \br{1 - \frac{M_f}{2}} - \frac{5 \omega L_{\max}}{2 \mu^2}$, $\omega \eqdef \frac{1}{p} - 1 < \frac{\mu^2}{10 L_{\max}^2}$, and $L_{\max} \eqdef \max_i L_i$.

If Lipschitzness and Assumption 5 are replaced with \emph{norm condition}: 
\begin{equation} \label{eq:norm_cond}
    \norm{\Exp{\nabla f_{i_k}(\cC(x^k)) \,|\, x^k} - \nabla f(x^k)} \leq \theta \norm{\nabla f(x^k)}    
\end{equation}
they obtain the following (for step size $\gamma=\nicefrac{1}{2 L_{\max}}$)
\begin{equation*}
    \min_{k \in \{1,\dots,K\}} \Exp{\sqN{\nabla f (x^k)}} \leq \frac{f(x^0) - f(x^\star)}{\alpha (K+1)} + \frac{1}{\alpha} \cdot \br{\frac{5 L_{\max} \omega}{2} \sqn{x^\star} + \frac{M}{4 L_{\max}}},
\end{equation*}
where $\alpha = \frac{1}{2 L_{\max}} \br{\frac12 - \theta - \frac{M_f}{2}} - \frac{5 \omega L_{\max}}{2 \mu^2}$ and $\omega = \frac{1}{p} - 1 < \frac{\mu^2}{5 L_{\max}^2 \br{\frac12 - \theta - \frac{M_f}{2}}}$.

\begin{remark}
The original method \eqref{eq:orig_IST_method} does not incorporate gradient sparsification, which can create a significant disparity between theory and practice. This is because the gradient computed at the compressed model, denoted as $\nabla f(\cC(x))$, is not guaranteed to be sparse and representative of the submodel computations. Such modification of the method also significantly simplifies theoretical analysis, as using a single sketch (instead of $\mC\mL\mC$) allows for an unbiased gradient estimator.

Through our analysis of the IST gradient estimator in Equation \eqref{eq:hom_gen_estimator}, we discover that conditions---such as Assumption 5 and Inequality \eqref{eq:norm_cond}---are not satisfied, even in the homogeneous setting for a simple quadratic problem. Furthermore, it is evident that such conditions are also not met for logistic loss. At the same time, in general, it is expected that insightful theory for general (non-)convex functions should yield appropriate results for quadratic problems. Additionally, it remains unclear whether the norm condition \eqref{eq:norm_cond} is satisfied in practical scenarios. The situation is not straightforward---even for quadratic problems---as we show in the expression for $\sigma^2$ in Equation \eqref{eq:hetero_bound}.
\end{remark}

\paragraph{Masked training \cite{mohtashami2022masked}.}

The authors consider the following \say{Partial SGD} method

\begin{equation*}
\begin{aligned}
   \hat{x}^k &= x^k + \delta x^k= x^k - (\mathrm{1} - p) \odot x^k\\
   x^{k+1} &= x^k - \gamma p \odot \nabla f(\hat{x}^k, \xi^k),
\end{aligned}
\end{equation*}

where $\nabla f(x, \xi)$ is an unbiased stochastic gradient estimator of a $L$-smooth loss function $f$, $\odot$ is an element-wise product, and $p$ is a binary sparsification mask.

\citet{mohtashami2022masked} make the following \say{bounded perturbation} assumption
\begin{equation} \label{eq:bound_perturb}
    \max_{k} \frac{\norm{\delta x^k}}{\max \left\{\norm{p^k \odot \nabla f(x^k)}, \norm{p^k \odot \nabla f(\hat{x}^k)}\right\}} \leq \frac{1}{2L}.
\end{equation}

This inequality may not hold for a simple convex case. Consider a function $f(x) = \frac{1}{2} x^\top A x$, for
\begin{equation*}
    A = \begin{pmatrix}
        a & 0 \\
        0 & c
        \end{pmatrix}, 
    \qquad
    x^0 = \begin{pmatrix}
        x_1 \\
        x_2
        \end{pmatrix}, 
    \qquad
    p^0 = \begin{pmatrix}
        0 \\
        1
        \end{pmatrix}.
\end{equation*}

Then condition \eqref{eq:bound_perturb} (at iteration $k=0$) will be equivalent to
\begin{equation*}
    \frac{x_1}{c x_2} \leq \frac{1}{2a} \Leftrightarrow 2 \leq \frac{2a}{c} \leq \frac{x_2}{x_1},
\end{equation*}
which clearly does not hold for an arbitrary initialization $x^0$.

In addition, the convergence bound in \cite[Theorem 1]{mohtashami2022masked} suggests choosing the step size as $\gamma_0 \alpha^k$, where 
\begin{equation*}
    \alpha^k = \min \left\{1, \frac{\left\langle p^k \odot \nabla f(x^k), p^k \odot \nabla f(\hat{x}^k)\right\rangle}{\left\|p^k \odot \nabla f(\hat{x}^k)\right\|^2}\right\}    
\end{equation*}
is not guaranteed to be positive to the inner product $\left\langle p^k \odot \nabla f(x^k), p^k \odot \nabla f(\hat{x}^k)\right\rangle$, which may lead to non-convergence of the method.

\paragraph{Optimization with access to auxiliary information framework \cite{chayti2022optimization}}
suggests modeling training with compressed models via performing gradient steps with respect to function $h(x) \eqdef \ExpSub{\mathcal{M}}{f(1_{\mathcal{M}} \odot x)}$. This function allows access to a sparse/low-rank version of the original model $f(x)$. They impose the following bounded Hessian dissimilarity assumption on $h$ and $f$

\begin{equation*}
    \left\|\nabla^2 f(x) - \ExpSub{\mathcal{M}}{\mD_{\mathcal{M}} \nabla^{2} f(1_{\mathcal{M}} \odot x) \mD_{\mathcal{M}}}\right\|_{2} \leq \delta,
\end{equation*}
where $1_{\mathcal{M}}$ and $\mD_{\mathcal{M}} = \mathrm{Diag}(1_{\mathcal{M}})$ refer to a binary vector and matrix sparsification masks.

This approach relies on variance-reduction and requires gradient computations on the full model $x$, and thus it is not suitable for our problem setting.

\paragraph{Comparison to the work of \citet{liao2022on}.}
Next, we try our best to briefly and accurately represent some of the previous work’s findings and comment on the differences.

The authors provide a \textit{high probability} convergence analysis of a \say{Single Hidden-Layer Neural Network with ReLU activations} based on the Neural Tangent Kernel (NTK) framework. The network’s first layer weights are initialized based on $\mathcal{N} (0, \kappa^2 \mathbf{I})$ and the weight vector of the second layer is initialized uniformly at random from $\{-1, 1\}$.
In contrast, we do not make any assumptions on the initialized parameters $x$ (in our notation).

The second differentiation is assumptions on the data. \citet{liao2022on} assume that for every data point $(a_j, y_j)$, it holds that $||a_j||^2 = 1$ and $|y_j| \leq C-1$ for some constant $C \geq 1$. Moreover, for any $j \neq l$, it holds that the points $a_i, a_l$ are not co-aligned, i.e., $a_i \neq \xi a_l$ for any $\xi \in \mathbb{R}$. In contrast, we do not make any assumptions about the data apart from the ones on matrices $\mathbf{L}_i$.
In addition, analysis by \citet{liao2022on} assumes that the number of hidden nodes is greater than a certain quantity and that NN’s weights distance from initialization is uniformly bounded.

\citet{liao2022on} consider a regression (MSE) loss function, a special case of quadratic loss and full gradients computation. They provide guarantees for IST under a \say{simplified assumption that every worker has full data access}, which corresponds to the homogeneous setting in our terminology.

\paragraph{Comparison of IST and 3D Parallelism \cite{shoeybi2019megatron}.}

IST and 3D parallelism were introduced independently and concurrently in 2019. While sharing some conceptual similarities (combination of data and model parallelism), they have a different focus. The key distinction is the way how model parallelism is implemented. Namely, 3D parallelism combines \textit{Pipeline} and \textit{Tensor} parallelism.

Pipeline parallelism suggests splitting the model’s layers across computing nodes which requires additional transmission of tensors for every forward and backward step. Tensor parallelism breaks the layers (MLP blocks and attention heads) into parts, which creates a need for additional synchronization between GPUs. Thus, 3D parallelism makes computation on nodes dependent on mutual communications, unlike IST.

Independent Subnetwork Training decomposes the model into smaller subnets independently trained in parallel. IST does not need synchronization during local updates and requires the transmission of fewer parameters, which decreases per-step communication costs.

3D parallelism introduces significant communication overhead, which greatly increases computational cost \cite{bian2023does}. The problem described is especially relevant for public compute clouds (such as Amazon EC2), which often suffer from slow interconnects. At the same time, IST is most beneficial for such a setup as it improves communication efficiency by design. Moreover, 3D parallelism is incompatible with a standard federated learning setting. At the same time, the IST-like approach is a viable technique, as every network is independent and can be trained on resource-constrained devices \cite{dun2023efficient}. In summary, both IST and 3D parallelism are viable approaches with pros and cons and are best suited for different scenarios.

\end{document}

%% file: commands.tex
\usepackage{mathtools, amsthm, amsmath, amssymb}
\mathtoolsset{showmanualtags}
\usepackage{xcolor}
\usepackage{tabularx}
\usepackage{bbm}
\usepackage{makecell}

\usepackage[utf8]{inputenc} 
\usepackage[T1]{fontenc}    
\usepackage{url}            
\usepackage{booktabs}       
\usepackage{amsfonts}       
\usepackage{nicefrac}       
\usepackage{microtype}      
\usepackage{xcolor}         

\usepackage{xspace}
\usepackage{xfrac}
\usepackage{verbatim}

\usepackage{algorithm}
\usepackage{algpseudocode}
\usepackage{todonotes}
\usepackage{subcaption}

\usepackage{nicefrac}

\allowdisplaybreaks[1]

\newtheorem{definition}{Definition}
\newtheorem{theorem}{Theorem}
\newtheorem{lemma}{Lemma}

\newtheorem{remark}{Remark}
\newtheorem{assumption}{Assumption}

\newcommand{\finf}{f^{\mathrm{inf}}}

\newcommand{\Diag}{\mathrm{Diag}}

\newcommand{\tr}{\mathrm{tr}}

\DeclareMathOperator{\bfC}{\mathbf{C}}
\DeclareMathOperator{\mcL}{\tilde{\mathbf{L}}}

\DeclareMathOperator{\cb}{\tilde{\mathrm{b}}}

\DeclareMathOperator{\moL}{\overline{\mathbf{L}}}
\DeclareMathOperator{\moB}{\overline{\mathbf{B}}}
\DeclareMathOperator{\moD}{\overline{\mathbf{D}}}
\DeclareMathOperator{\moW}{\overline{\mathbf{W}}}
\DeclareMathOperator{\ob}{\overline{b}}
\DeclareMathOperator{\oCb}{\overline{\mathbf{C}b}}

\DeclareMathOperator{\bb}{\mathrm{b}}
\DeclareMathOperator{\oDb}{\overline{\mD^{-\frac{1}{2}}\bb}}
\newcommand{\cDb}{\widetilde{\mD\bb}}

\DeclareMathOperator{\E}{\mathbf{E}}

\newcommand{\Exp}[1]{{\color{black}  \mathbb{E}}\left[#1\right]}
\newcommand{\ExpSub}[2]{{\mathbb{E}}_#1\left[#2\right]}

\newcommand{\Rd}{\mathbb{R}^d}
\newcommand{\R}{\ensuremath{\mathbb{R}}}

\newcommand{\eps}{\varepsilon}

\def\<#1,#2>{\langle #1,#2\rangle}

\newcommand{\norm}[1]{\|#1\|}
\newcommand{\sqn}[1]{\norm{#1}^2}
\newcommand{\Norm}[1]{\left\|#1\right\|}
\newcommand{\sqN}[1]{\Norm{#1}^2}

\newcommand\br[1]{\left ( #1 \right )}
\newcommand\sbr[1]{\left[#1\right]}
\newcommand\cbr[1]{\left\{#1\right\}}

\newcommand{\cC}{\mathcal{C}}
\newcommand{\cQ}{\mathcal{Q}}

\newcommand{\cO}{\mathcal{O}}

\newcommand{\mA}{\mathbf{A}}
\newcommand{\mB}{\mathbf{B}}
\newcommand{\mC}{\mathbf{C}}
\newcommand{\mD}{\mathbf{D}}

\newcommand{\mW}{\mathbf{W}}
\newcommand{\mL}{\mathbf{L}}

\newcommand{\mI}{\mathbf{I}}

\newcommand{\bU}{\mathbb{U}}

\newcommand{\bR}{\mathbb{R}}

\newcommand{\eqdef}{\coloneqq}
\DeclareMathOperator*{\argmin}{arg\,min}